\documentclass[11pt]{preprint}

\usepackage{times}
\usepackage{latexsym}
\usepackage[T1]{fontenc}
\usepackage[utf8]{inputenc}
\usepackage{microtype}

\usepackage{graphicx}
\usepackage{booktabs}
\usepackage{tabularx}
\usepackage{multirow}
\usepackage{subcaption}
\usepackage{amsmath}
\usepackage{amssymb}
\usepackage{amsthm}
\usepackage{xcolor}
\usepackage{colortbl}
\usepackage{enumitem}
\usepackage{needspace}
\usepackage{adjustbox}
\usepackage[capitalize,noabbrev]{cleveref}

\newtheorem{proposition}{Proposition}
\newtheorem{lemma}{Lemma}

\newtheorem{assumption}{Assumption}
\crefname{property}{Property}{Properties}
\Crefname{property}{Property}{Properties}
\crefname{assumption}{Assumption}{Assumptions}
\Crefname{assumption}{Assumption}{Assumptions}

\newcommand{\FL}{\mathrm{FL}}
\definecolor{provisional}{HTML}{1F5FBF}

\definecolor{ourfix}{HTML}{006400}

\newcommand{\opt}[1]{{\small\textsf{#1}}}
\newcommand{\spear}{r_s}

\title{One Bias Score, Two Factors: How Mitigations Affect Option Bias in Multiple-Choice Evaluation}

\author{Hyunjong Ok}
\author{Minjae Park}
\author{Jaeho Lee}
\affiliation{Pohang University of Science and Technology (POSTECH)}
\contribution{\texttt{\{hyunjong.ok, mjae.park, jaeho.lee\}@postech.ac.kr}}
\usepackage{etoc}

\usepackage{placeins}

\renewcommand{\dbltopfraction}{0.85}
\begin{document}
\etocdepthtag.toc{mtbody}
\maketitle

\begin{abstract}
Reordering multiple-choice options can change a model's answer---an effect studied as \emph{selection bias} (a preference for certain option letters) and as \emph{order sensitivity} (changes in answers across option orderings). We connect these perspectives through the pooled selection-rate standard deviation (pooled SD), which measures how unevenly option letters are selected across questions and orderings. When every option of each question appears equally often at each position, the pooled SD is exactly the product of two label-free factors: the \emph{flip rate} and the \emph{positional skew}. The flip rate is the fraction of questions whose chosen option changes across orderings; the positional skew is the standard deviation of option-letter selection rates among those questions. Across 41 checkpoints and 12 benchmarks, the flip rate tracks accuracy and output uncertainty, while the positional skew tracks the estimated option-symbol prior. Prior correction for selection bias mainly reduces the positional skew. Aggregating each ordering with its reversal, a mitigation for order sensitivity, mainly reduces the flip rate. Both reduce the range of accuracy across correct-answer letters by similar amounts, but aggregation leaves fewer model-rank reversals. Applying prior correction before aggregation reduces both factors and further narrows this accuracy range. However, aggregation lowers the flip rate partly by settling questions on the same wrong option. We therefore recommend reporting both factors to distinguish mitigation effects, alongside accuracy to distinguish stability from correctness.
\end{abstract}

\section{Introduction}
\label{sec:intro}
\begin{figure*}[t]
\centering
\begin{subfigure}[b]{0.135\textwidth}
\includegraphics[width=\linewidth]{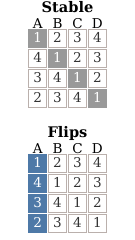}
\subcaption{What a flip is.}\label{fig:concept-a}
\end{subfigure}\hfill
\begin{subfigure}[b]{0.55\textwidth}
\includegraphics[width=\linewidth]{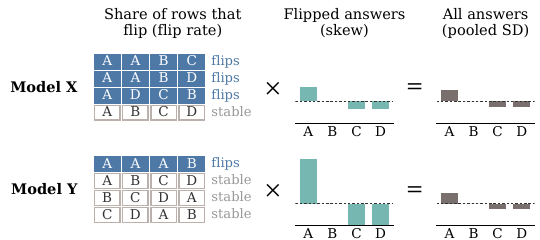}
\subcaption{The same pooled SD from different behaviours.}\label{fig:concept-b}
\end{subfigure}\hfill
\begin{subfigure}[b]{0.295\textwidth}
\includegraphics[width=\linewidth]{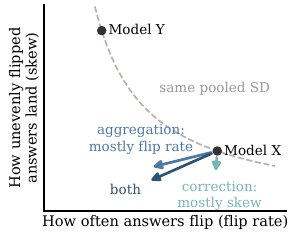}
\subcaption{The factor each method moves.}\label{fig:concept-c}
\end{subfigure}
\caption{\textbf{The same pooled SD can come from different behaviours, and mitigation methods lower it
through different factors.} (a) An answer that stays on one option lands once on every letter; one that
stays on letter \opt{A} changes option each time, a flip. (b) Rows are questions; bars give each
letter's share of flipped and of all answers from the dashed $1/4$. Two constructed models reach the same pooled SD:
Model X flips often with a mild lean to \opt{A}, Model Y rarely but strongly. (c) Each method's measured
mean direction; the dashed curve joins equal pooled SDs.}
\label{fig:concept}
\end{figure*}

Multiple-choice evaluation of language models depends on the order in which the options are presented. Multiple-choice benchmarks rank models by how often they select the correct option, assuming that a model's choice depends on what the options say, not on where they appear. Yet reordering the options, without changing which answer is correct, can change model selections enough to alter these rankings \citep{alzahrani-etal-2024-benchmarks}. 

Two concerns organise the work on this effect, which we call option bias. \textit{Selection bias} is a preference for certain options by their letter (the option symbol: \opt{A}, \opt{B}, \opt{C}, \ldots); it is summarised with statistics pooled over a benchmark's answers and reduced by \textit{prior correction}, which divides out an estimated option-symbol prior \citep{zheng2024large,reif-schwartz-2024-beyond}, or by intervening on the model \citep{choi-etal-2025-mitigating}. \textit{Order sensitivity} is the change of an answer with the order; it is counted per item \citep{pezeshkpour2024large,li2024perteval} and handled by \textit{aggregation} over several orderings \citep{pezeshkpour2024large,wei-etal-2024-unveiling}. Selection bias is thus read over pooled answers, order sensitivity per question.
\looseness=-1

Mitigation methods are commonly compared by a \textit{pooled bias score} \citep[e.g.,][]{zheng2024large,choi-etal-2025-mitigating}: a single number computed over all of a benchmark's answers that grows as the answers concentrate on some letters or as accuracy varies across them. One such score is the \textit{pooled selection-rate standard deviation} (\textit{pooled SD}), the standard deviation across letters of each letter's share of the answers. Two methods with similar gains can change different parts of option bias. Take one method from each line of work. A prior correction divides each letter's probability by an estimate of how much the model favours that letter whatever option it holds \citep{zheng2024large}; an aggregation asks each question in a given order and in its reverse and picks the option whose two probabilities have the largest product \citep{wei-etal-2024-unveiling}. Both roughly halve how much accuracy changes with the letter holding the correct answer, yet only the aggregation makes answers change less often when the options are reordered: it removes about a third of the questions whose answer changes, while the correction removes almost none. It also leaves fewer reversals of model rankings than the correction, and it lowers flips partly by settling questions on a wrong answer.
\looseness=-1

The two methods narrow this accuracy gap alike yet change different behaviour because a pooled bias score mixes two things: how often answers change with the order, and how the changing answers lean. Many such scores start from how often each letter is chosen, or how much probability it receives \citep{zheng2024large,pezeshkpour2024large,li-etal-2024-multiple,choi-etal-2025-mitigating}. A model blind to the order would spread its answers evenly, a quarter on each of four letters, while one that favours \opt{A} puts more on \opt{A}, raising the pooled SD. Now ask every question once per rotation of its options, \opt{A~B~C~D}, then \opt{D~A~B~C}, and so on, so that each option appears once at every position, a \textit{balanced} set of orderings. A question whose answer stays on the same option lands once on every position, whereas one answered \opt{A} under every rotation picks a different option each time and puts all its answers on \opt{A} (\cref{fig:concept-a}). Only changing answers make the shares uneven.

From this, the pooled SD is exactly the product of two factors that need no gold labels: the \textit{flip rate}, the
share of questions whose chosen option changes across rotations, and the \textit{positional skew}, how unevenly those questions' answers fall on the letters. Two models can share a pooled SD, one with many flips spread nearly
evenly and one with few concentrated on particular letters (\cref{fig:concept-b}).

The two factors show which behaviour a model or a mitigation method has changed: on their plane, prior correction moves a model mostly along the skew and the aggregation mostly along the flip rate (\cref{fig:concept-c}). The selection-bias statistic RStd \citep{zheng2024large}, the spread of accuracy over the letters that hold the correct option, splits the same way under the same orderings, into the flip rate times how unevenly the flipping questions' accuracy spreads over those letters. A consistency count says nothing about where changing answers go; a pooled bias score mixes both factors. Prior work reports consistency beside position preference \citep{guda2025quantifying,shi-etal-2025-judging}. We make this pairing exact and, since neither statistic alone shows which factor a mitigation changed, propose reporting both factors with accuracy.

\paragraph{Contributions.}
\begin{enumerate}[leftmargin=1.4em,itemsep=2pt,topsep=2pt,parsep=0pt]
\item \textbf{An exact, label-free factorisation} of the pooled SD into the flip rate and positional skew under balanced orderings, extending to RStd.
\item \textbf{An operation-level account of mitigation}: prior correction mainly lowers the skew and
aggregation over an ordering and its reversal mainly the flip rate. The two narrow the accuracy gap across correct-answer letters by similar amounts but leave different rank reversals and wrong answers, and they compose.
\item \textbf{Evidence for what each factor captures} across 41 checkpoints and 12 benchmarks: the flip rate tracks accuracy and output uncertainty, and the skew the model's option-symbol prior.
\end{enumerate}

\section{Related Work}
\label{sec:related}

We discuss the most relevant literature here and provide a broader discussion in \cref{app:related}.

\paragraph{Selection bias and prior correction.}
\citet{zheng2024large} define selection bias as a preference for certain option symbols (option IDs in their terms), identify a prior over
ID tokens as a primary contributor alongside model- and task-dependent position effects, and measure it with
RStd, the SD of classwise recalls, under the premise that correct answers are randomly placed. Their PriDe
estimates the prior from permutations of a few items and divides it out of the rest; they combine it with
partial permutation averaging and relate the estimated prior to option-position recalls. Related
corrections prune bias nodes or add an auxiliary option
\citep{choi-etal-2025-mitigating}, intervene on attention heads and feed-forward vectors
\citep{zhou2024unibias}, or correct label priors in in-context classification, where content-free estimates
can hurt \citep{fei2024mitigating,reif-schwartz-2024-beyond}. \citet{guda2025quantifying} observe that, on average,
PriDe lowers RStd but raises other measures, evaluating pooled and per-item measures together. We
distinguish the quantities measured and the operations applied rather than assigning each paper one
perspective, and show which factor of a pooled score a correction moves.

\paragraph{Order sensitivity and aggregation.}
Order sensitivity is conjectured to arise from an interaction between uncertainty among a model's
leading choices and question-dependent positional preferences, with the uncertainty part supported by GPT-4
answer probabilities and accuracy improved by a majority vote over reorderings
\citep{pezeshkpour2024large}.
\citet{wei-etal-2024-unveiling} measure a two-ordering fluctuation rate, find more fluctuation on subtasks with
lower accuracy in most tested models, and multiply each content's probabilities under an ordering and
its reversal; they also study probability calibration. Inconsistency under answer-preserving perturbations has been read as a sign that single-format scores
overstate knowledge and built into scoring \citep{li2024perteval,gupta2024changing}, or reported as
reliability beside accuracy \citep{cocchieri-etal-2026-remedqa}.
These measures count changes under their chosen perturbations; we relate the flip rate to pooled scores under a balanced design and show that which orderings an aggregation pairs changes whether it also
lowers the skew.

\paragraph{Reporting both quantities.}
LLM-judge audits report how often an order effect occurs beside which position it favours
\citep{shi-etal-2025-judging,wang-etal-2024-large-language-models-fair}, and for two candidates the raw
fairness score of \citet{shi-etal-2025-judging} already combines the two. We give the exact factorisation of the pooled SD under any balanced ordering design for multiple-choice evaluation and use it to
read what mitigation operations change.
\looseness=-1

\section{Two factors in pooled bias statistics}
\label{sec:method}

\subsection{Cyclic reordering and the flip rate}
\label{sec:method_protocol}
\label{sec:method_stats}

\begin{figure*}[t]
\centering
\includegraphics[width=0.9\linewidth]{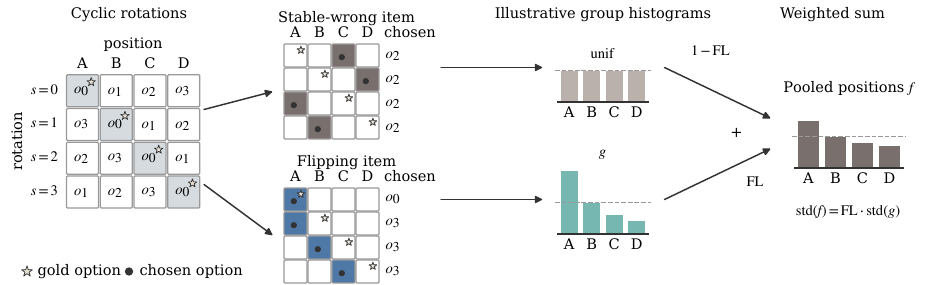}
\caption{\textbf{An item that does not flip lands once on every position, so only flipping items can make the pooled shares uneven.} Left: rotation moves every content through every position. Middle: a stable-wrong item's answers spread evenly; a flipping item's answers pile onto early letters. Right: pooling weights the even share by $1-\FL$ and the flipping items' share $g$ by $\FL$. Stars mark the gold option, dots the choice and dashes $1/K$; the histograms are illustrative.}
\label{fig:protocol}
\end{figure*}

We ask every item once per rotation of its options and record which option the model chooses. Fix a
checkpoint and a benchmark of $N$ items. Each item $q$ has $K$ answer options, indexed by the benchmark's
canonical order. An option's \textit{content} is its text, wherever it is shown; let $o_j$ denote the content of
option $j$, and $c_q$ the index of the correct option. For rotation $s \in \{0, \ldots, K-1\}$, the
cyclic ordering $\pi_s$ puts content $o_j$ at position
\begin{equation}
  \pi_s(j) = (j + s) \bmod K,
  \label{eq:shift}
\end{equation}
where positions $0, 1, \ldots, K-1$ are the letters \opt{A}, \opt{B} and so on. Thus $\pi_0$ is the canonical order, and over
the $K$ orderings every content visits every position once. Only the order changes; the question, options and correct option stay fixed (\cref{fig:protocol}).

Let $z_q(\pi)$ be the index of the content chosen under ordering $\pi$, and $a_q(\pi) = \mathbf{1}[z_q(\pi) =
c_q]$ its correctness. Accuracy $A = \frac{1}{NK}\sum_{q}\sum_{s} a_q(\pi_s)$ averages correctness over all
items and all orderings in the cyclic set $S_{\mathrm{cyc}} = \{\pi_0, \ldots, \pi_{K-1}\}$, not only the
canonical one. The \textit{flip rate} is the share of items whose chosen content changes with the ordering:
\begin{equation}
  \FL = \tfrac{1}{N} \bigl|\{q : z_q \text{ non-constant on } S_{\mathrm{cyc}}\}\bigr| .
  \label{eq:fl}
\end{equation}
A change is counted by content, not by letter: an item answered \opt{A} under every ordering flips, as \opt{A} holds
a different content each time. The flip rate needs no gold label. Each item is thus flipping or stable; a stable item is \textit{stable-correct} if its choice is correct and \textit{stable-wrong}
otherwise.

Accuracy limits how many items can flip. The $NK$ answers include $NK(1 - A)$ wrong ones, and every flipping
item gives at least one of them, so
\begin{equation}
  \FL \le \min\bigl(1, K(1 - A)\bigr) ,
  \label{eq:fl_cap}
\end{equation}
which we call the \textit{flip cap}. In turn, the flip rate bounds how much accuracy can change with the order:
a stable item is right under every cyclic ordering or wrong under all of them, so the accuracies under any
two cyclic orderings differ by at most the flip rate. The flip rate over the cyclic orderings is also a lower bound on
its value over all $K!$ orderings (\cref{app:orderings}).

\subsection{The pooling identity}
\label{sec:method_factors}

Selection bias shows up in where the answers land. Let $f_q(r)$ be the share of item $q$'s $K$ answers that
fall on position $r$, and $f(r) = \frac{1}{N}\sum_q f_q(r)$ the share of all answers at $r$. The \textit{pooled selection-rate SD}, or pooled SD, is
\begin{equation}
  \operatorname{std}(f) = \Bigl(\tfrac{1}{K}\textstyle\sum_r \bigl(f(r) - \tfrac{1}{K}\bigr)^2\Bigr)^{1/2} ,
  \label{eq:pooled_sd}
\end{equation}
which is zero when the answers are spread evenly and grows as they pile onto some positions. \citet{zheng2024large}
consider these per-position counts a simple measure of selection bias but set them aside, because a
benchmark may place its correct answers unevenly over the letters and so tilt the counts without any bias.
Balanced orderings place the correct answers evenly and remove this concern. Related statistics average
probabilities instead of counts \citep{pezeshkpour2024large} or compare predicted and gold distributions \citep{choi-etal-2025-mitigating}.

Under the cyclic orderings, an item that does not flip cannot skew $f$ (\cref{fig:protocol}). Such an item
chooses the same content $o_{z_q}$ under every ordering, and ordering $\pi_s$ shows it at position
$(z_q + s) \bmod K$, so its $K$ answers land once on every position and $f_q$ is the uniform share
$\mathrm{unif}(r) = 1/K$. Let $\mathcal{F}$ be the set of flipping items and
$g = \frac{1}{|\mathcal{F}|}\sum_{q \in \mathcal{F}} f_q$ the share of their answers at each position (defined when
some item flips).
Averaging $f_q$ separately over the stable and the flipping items, which make up shares $1 - \FL$ and $\FL$
of all items, gives the pooled share at each position
\begin{equation}
  f = (1 - \FL)\,\mathrm{unif} + \FL \cdot g .
  \label{eq:pooling_mixture}
\end{equation}
\textbf{Pooling identity.} Subtracting $\mathrm{unif}$ from both sides gives $f - \mathrm{unif} = \FL\,(g - \mathrm{unif})$, so the SD
of $f$ factors as
\begin{equation}  \underbrace{\operatorname{std}(f)}_{\text{pooled SD}}
  = \underbrace{\FL}_{\text{flip rate}} \cdot
  \underbrace{\operatorname{std}(g)}_{\text{positional skew}} .
  \label{eq:rstd_factor}
\end{equation}
The first, the flip rate, is how often order changes an
answer. The second, the \textit{positional skew} $\operatorname{std}(g)$, is how unevenly the flipping items'
answers spread over positions, and $g - \mathrm{unif}$ shows which positions they favour. Neither needs gold labels.

The identity is arithmetic: it says how the pooled SD is composed, not why a model flips or leans. The same
pooled SD can come from many flips with a mild lean or from few flips with a strong one. Two of our
checkpoint--benchmark pairs, which we call cells, have nearly the same pooled SD at opposite
combinations of the two factors
(\cref{fig:decomposition}). Flips are needed for an uneven $f$ but do not guarantee one, since flipping items' leans can cancel.

\begin{figure*}[t]
\centering
\includegraphics[width=0.9\linewidth]{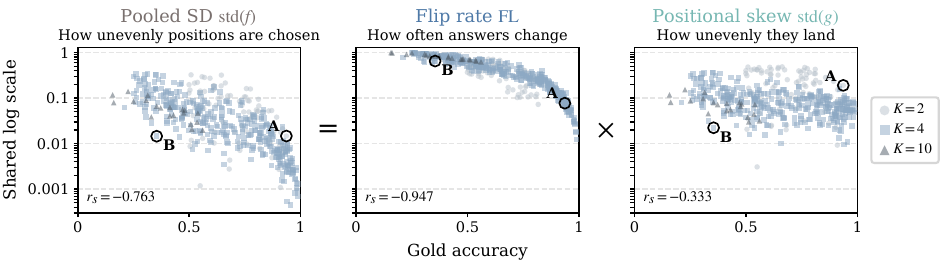}
\caption{\textbf{The pooled SD is the flip rate times the skew; the flip rate tracks accuracy, the skew only weakly.} Points are cells, coloured by option count $K$; $\spear$ is Spearman's correlation with accuracy. Cells A (Qwen2.5-72B-Instruct, PIQA) and B (Llama-3.2-3B, LogiQA) reach nearly the same pooled SD from opposite factor combinations.}
\label{fig:decomposition}
\end{figure*}

The identity needs only that every content appears at every position equally often, so it holds for any
such balanced set of orderings, including all $K!$. Under an unbalanced set a stable item can also skew $f$;
subtracting the stable items' deviation from uniform, weighted by their share, restores the identity
without gold labels
(\cref{app:any_design}).

\subsection{What each kind of statistic reads}
\label{sec:method_metrics}

Gold-label statistics such as RStd ask instead how accuracy depends on the letter that holds the correct
option, and they split the same way. Let $A_j$ be the accuracy when every item is shown in the cyclic ordering
that puts its correct option at letter $j$. Each item has exactly one such ordering, in which a
stable-correct item is right and a stable-wrong item wrong, so the accuracy at letter $j$ is
\begin{equation}
  A_j = \mathrm{SC} + \FL \cdot h_j ,
  \label{eq:placement}
\end{equation}
with $\mathrm{SC}$ the share of stable-correct items and $h_j$ the accuracy of the flipping items when their
correct option is at letter $j$ (likewise defined when some item flips). RStd \citep{zheng2024large} is the SD of $A_j$ over the letters; computed
under balanced orderings, as here, we call it \textit{balanced RStd}, and it equals $\FL \cdot
\operatorname{std}(h)$. The \textit{placement range}, $\max_j A_j - \min_j A_j$, is how much accuracy can change
when every correct option is moved to one letter, the cyclic form of the answer-moving attack of
\citet{zheng2024large}; it equals $\FL \cdot (\max_j h_j - \min_j h_j)$. The vector $g$ records where the flipping
items' chosen answers land, and $h_j$ their accuracy when the correct option is at letter $j$. With two options
$h$ and $g$ coincide; with more, their SDs rank cells similarly, but neither vector determines the other (\cref{app:formal_skew_placement}).

\begin{wraptable}{R}{0.48\textwidth}
\setlength{\columnwidth}{\linewidth}
\centering
\small
\setlength{\tabcolsep}{24pt}
\begin{tabular}{@{}lc@{}}
\toprule
Statistic & Value \\
\midrule
\multicolumn{2}{@{}l}{\textit{Order sensitivity}} \\
Flip rate & $\FL$ \\
Two-order disagreement & ${\le}\,\FL$ \\
\addlinespace
\multicolumn{2}{@{}l}{\textit{Selection bias}} \\
Pooled SD & $\FL \cdot \operatorname{std}(g)$ \\
Balanced RStd & $\FL \cdot \operatorname{std}(h)$ \\
Forward KL to uniform & ${\le}\,\FL \cdot D_{\mathrm{KL}}(g \,\|\, \mathrm{unif})$ \\
Placement range & $\FL \cdot \operatorname{range}(h)$ \\
\bottomrule
\end{tabular}
\caption{\textbf{Consistency statistics count changing answers; the pooled SD and balanced RStd factor through the flip rate, and the forward KL is bounded through it.} Values under balanced orderings ($\le$: bound for two evaluated orders); $g$ is where flipping answers land and $h_j$ their accuracy with the correct option at $j$. Only the pooled SD splits into two label-free factors.}
\label{tab:metric_map}
\end{wraptable}

\Cref{tab:metric_map} sorts the statistics by what they read. Consistency statistics count items whose answer
changes. The flip rate counts a change between any of the evaluated orderings; a two-order disagreement rate,
such as the fluctuation rate of \citet{wei-etal-2024-unveiling}, counts only changes between its two orders,
so it is at most the flip rate when both orders are among those evaluated. Neither says where the changing
answers go. The pooled SD and balanced RStd multiply the flip rate by a lean of the flipping items: the skew
$\operatorname{std}(g)$ for the pooled SD, read from predictions alone, and $\operatorname{std}(h)$
for balanced RStd, which needs gold labels. Forward KL is only bounded,
$D_{\mathrm{KL}}(f \,\|\, \mathrm{unif}) \le \FL \cdot D_{\mathrm{KL}}(g \,\|\, \mathrm{unif})$. None of
these statistics reads the lean of the changing answers on its own. A pooled bias score moves with either factor without telling which, and a consistency count sees only the first;
reading the flip rate and the skew together separates them. \Cref{app:formal} proves these results.
\looseness=-1

\section{Experimental setup}
\label{sec:setup}

\paragraph{Models and benchmarks.}
We evaluate 41 open-weight checkpoints from seven model families, pretrained and post-trained, on 12
benchmarks with two, four or ten options, such as PIQA \citep{bisk2020piqa}, a five-subject
subset of MMLU \citep{hendrycks2021measuring} and MMLU-Pro \citep{wang2024mmlupro}. A \emph{cell} is one checkpoint--benchmark pair; correlations across cells count each once.

\paragraph{Readouts.}
The answer is the option letter with the highest next-token probability, the \emph{first-token readout}.
Five-shot prompting, and direct-answer and chain-of-thought prompting \citep{wei2022chainofthought} with
greedy decoding, serve as checks.

\paragraph{Mitigation operations.}
The operations rework the stored letter probabilities and are compared on the same items. \emph{Prior
correction} divides every query's letter probabilities by an option-symbol prior, the model's preference for
each letter, estimated on a held-out 5\% of a cell's items as in PriDe's reported setting: PriDe's estimate
\citep{zheng2024large} or the mean letter probability, a variant of batch calibration \citep{zhou2024batch}. \emph{Aggregation} multiplies, for each
content, its probabilities under the presented ordering and under its reversal, the probability product of
\citet{wei-etal-2024-unveiling}, which we call the \emph{reversal product}, and returns the content with the
largest product. A \emph{composition} applies two operations in sequence, chiefly a correction and then the product. Every operation returns one
answer per cyclic ordering, and both factors are computed over these answers.

\paragraph{Statistics.}
Correlations are Spearman ($\spear$) unless stated otherwise, and differences between conditions on the same
cells carry 95\% checkpoint-block intervals, which resample whole checkpoints. \Cref{app:pools} gives the
full protocol.

\section{What the two factors capture}
\label{sec:factors}

\begin{figure*}[t]
\begin{minipage}[t]{0.48\textwidth}\vspace{0pt}
\setlength{\columnwidth}{\linewidth}
\centering
\includegraphics[width=0.8\linewidth]{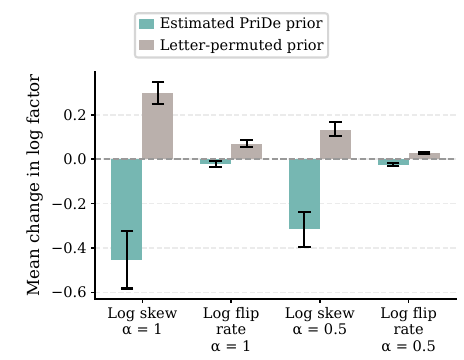}
\captionof{figure}{\textbf{Dividing out the model's own prior lowers the skew; the same prior with permuted letters raises it.} Mean change in log skew and log flip rate on held-out items after dividing by PriDe's prior raised to $\alpha$, or by the letter-permuted prior.}
\label{fig:factors}
\end{minipage}\hfill
\begin{minipage}[t]{0.48\textwidth}\vspace{0pt}
\setlength{\columnwidth}{\linewidth}
\centering
\includegraphics[width=0.85\linewidth]{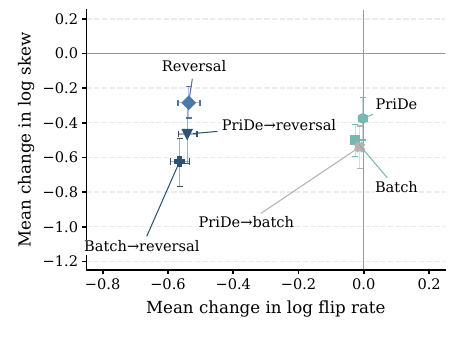}
\captionof{figure}{\textbf{Prior correction mainly lowers the skew; the reversal product mainly lowers the flip rate and part of the skew.} Mean change in log flip rate and log skew from no operation on the same held-out items of cells with at least three options; X$\to$Y applies X and then Y. Their sum is the change in log pooled SD.}
\label{fig:operations}
\end{minipage}
\end{figure*}

This section shows that the two factors of the identity record different behaviour.

\subsection{Flips track accuracy and uncertainty}
Accurate models flip less. Across cells, the flip rate falls with accuracy ($\spear = -0.95$), far more
closely than the pooled SD or the skew does (\cref{fig:decomposition}).
This is not only because accuracy caps the flip rate (\cref{eq:fl_cap}): the association holds where the cap
restricts nothing, and when accuracy is scored on items the flip rate does not read.
Within a cell, flips concentrate on items whose two leading options are close in probability under the
canonical order, a sign of output uncertainty: a smaller gap marks flipping rather than stable items with a
median AUC of $0.85$. This holds even among items wrong
under every ordering, so it is not accuracy in disguise. \citet{pezeshkpour2024large} conjecture such
uncertainty behind order sensitivity; whether the model's internal uncertainty causes the flips is left to
future work.
A low flip rate still does not mean high accuracy: an item that keeps one wrong answer does not flip.

\subsection{The skew reflects option-symbol bias}
Option-symbol bias is a model's preference for particular option letters, which PriDe estimates as a prior over
the letters \citep{zheng2024large}. We estimate this prior on the held-out 5\% of a cell's items and measure
the skew on the rest, and three results tie the two together. Models with a more uneven prior have a larger
skew, also given accuracy. The flipping answers lean towards the letters the
prior favours (median centred cosine $0.87$ on four-option cells, against $-0.10$ for the prior with its
letters permuted).
And dividing the prior out lowers the skew while the flip rate barely moves, whereas dividing out the permuted
prior raises it (\cref{fig:factors}). The skew is where a model's option-symbol bias shows in the changing answers; under cyclic orderings this letter preference is also a position preference.

Both factors are reliable: each ranks the cells nearly the same way on two random halves of their items, and
the two carry comparable shares of the variance of the log pooled SD across cells.
\Cref{app:factor_support} gives these analyses, the canonical-order gap and the controls on the flip rate.

\section{Mitigations move different factors}
\label{sec:operations}

\begin{figure*}[t]
\centering
\begin{subfigure}[t]{0.245\textwidth}
\includegraphics[width=\linewidth]{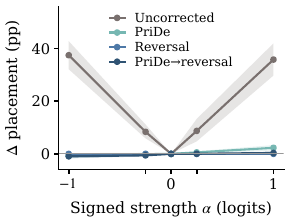}
\subcaption{\textbf{Front to back cancels.}}\label{fig:injection-a}
\end{subfigure}\hfill
\begin{subfigure}[t]{0.245\textwidth}
\includegraphics[width=\linewidth]{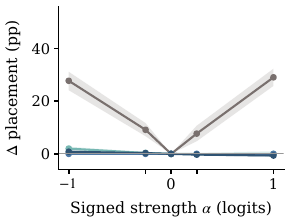}
\subcaption{\textbf{Alternating cancels.}}\label{fig:injection-b}
\end{subfigure}\hfill
\begin{subfigure}[t]{0.245\textwidth}
\includegraphics[width=\linewidth]{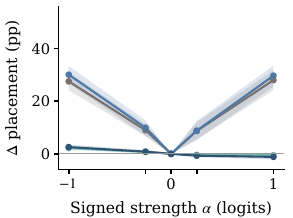}
\subcaption{\textbf{Outer versus inner passes.}}\label{fig:injection-c}
\end{subfigure}\hfill
\begin{subfigure}[t]{0.245\textwidth}
\includegraphics[width=\linewidth]{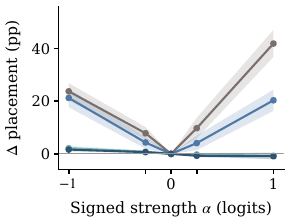}
\subcaption{\textbf{Letter \opt{A} alone partly passes.}}\label{fig:injection-d}
\end{subfigure}
\caption{\textbf{The reversal product removes an added letter preference only where mirrored letters differ, while PriDe largely removes each tested shape.} Change in placement range after adding $\alpha d_i$ to the logit of letter $i$ in every stored query of the MMLU cells, where $d$ is the panel's shape and a negative $\alpha$ reverses it; priors are re-estimated after the addition, and shading gives 95\% intervals. Under mirroring (\opt{A} with \opt{D}, \opt{B} with \opt{C}) the first two shapes cancel, the third passes through and the fourth has both parts.}
\label{fig:injection}
\end{figure*}

Prior correction mainly lowers the skew, and the reversal product mainly the flip rate together with part of
the skew (\cref{fig:operations}): correction changes mainly where changing answers land, the product mainly
whether answers change at all. Two exact properties help read
these patterns, stated for a \emph{letter preference}, an amount added to the log score of whatever content a
letter shows.

\paragraph{Correction acts within each ordering.}
Dividing by an estimated prior subtracts its log from each letter's log score in every query, but nothing
ties the answer under one ordering to the answer under another. In our experiments PriDe and batch calibration
lower the skew much more than the flip rate, also on the cells used to compare the corrections alone, which include
two-option cells.
There the small net change in the flip rate hides movement both ways: under PriDe some items stop
flipping and somewhat fewer start, while fewer of the previously flipping items' answers land on the prior's
favoured letter.
In PriDe's model an item's letter prior does not depend on the ordering, so removing it exactly would make the
answers agree, barring ties \citep{zheng2024large}; the correction instead uses one prior averaged over the estimation items,
and we do not separate estimation error, item-specific preferences and other departures from that model as
causes of the remaining flips.

\paragraph{The product averages mirrored letters.}
Multiplying a content's probabilities under an ordering and its reversal adds its two log scores. Reversal
moves the $i$th of $K$ letters to the $(K{+}1{-}i)$th, so the content receives the preferences of a letter and
its mirror together, \opt{A} with \opt{D} and \opt{B} with \opt{C} for four options. The part of a preference with opposite signs on
mirrored letters, such as a steady drift from \opt{A} to \opt{D}, cancels exactly; the part that is the same on them, such
as outer over inner letters, survives (\cref{app:operations_cancel}). Combining two orderings can also leave
fewer answers that change with the order, which the product does not guarantee but which we observe. Adding a
preference of known shape and strength $\alpha$ to every stored logit of the MMLU cells confirms these properties for added output preferences
(\cref{fig:injection}): front-to-back and alternating preferences leave every answer of the product unchanged,
an outer-versus-inner preference passes through it and widens its placement range, and PriDe, with its prior re-estimated, largely removes each tested shape. Models' own preferences
mix both parts: on four-option cells the part surviving mirroring keeps on average about half of a
preference's SD, and few preferences are monotone. This helps explain why the
product lowers the skew only partly; across cells, the larger the surviving part, the less the skew falls on
held-out items.

\paragraph{One statistic hides which factor moved.}
PriDe and the reversal product lower balanced RStd alike, from 0.064 to 0.036, but the product lowers the flip
rate far more. Judged by balanced RStd alone the two look similar; judged by the flip rate alone PriDe would seem
to change nothing, although it narrows the placement range about as much as the product
(\cref{sec:composition}). \Cref{app:operations} details the operations and the added preferences.
\looseness=-1

\section{Residual harms and composition}
\label{sec:composition}

\begin{table*}[t]
\centering
\small
\setlength{\tabcolsep}{8pt}
\resizebox{\textwidth}{!}{%
\begin{tabular}{lrrrrrr}
\toprule
Method & $\Delta\log\mathrm{FL}$ & $\Delta\log\mathrm{Skew}$ & \shortstack{Placement\\range (pp)} & \shortstack{Rank reversals\\(\% of pairs)} & \shortstack{$\Delta$ Accuracy\\(pp)} & \shortstack{$\Delta$ Stable-wrong\\(pp)} \\
\midrule
Baseline & 0.00 & 0.00 & 16.7 & 15.4 & 0.0 & 0.0 \\
PriDe \citep{zheng2024large} & $-0.00$ & $-0.38$ & 9.6 & 15.4 & 0.5 & 0.3 \\
Batch calibration, variant \citep{zhou2024batch} & $-0.03$ & $-0.50$ & 8.2 & 15.4 & 0.5 & 0.5 \\
Reversal product \citep{wei-etal-2024-unveiling} & $-0.54$ & $-0.28$ & 8.2 & 10.6 & 2.8 & 6.2 \\
PriDe then reversal & $-0.54$ & $-0.47$ & 5.3 & 10.6 & 2.9 & 6.6 \\
Batch then reversal & $-0.56$ & $-0.62$ & 4.6 & 10.5 & 2.9 & 6.8 \\
PriDe then batch & $-0.01$ & $-0.54$ & 7.7 & 15.8 & 0.5 & 0.4 \\
\midrule
\multicolumn{7}{l}{Paired differences on the same cells} \\
Comparison & \multicolumn{2}{c}{$\Delta\log$ pooled SD} & \multicolumn{2}{c}{$\Delta$ Placement range (pp)} & \multicolumn{2}{c}{\shortstack{$\Delta$ Rank reversals\\(pp of pairs)}} \\
Reversal product $-$ PriDe & \multicolumn{2}{c}{$-0.44$ [$-0.61$, $-0.26$]} & \multicolumn{2}{c}{$-1.3$ [$-2.5$, 0.0]} & \multicolumn{2}{c}{$-4.8$ [$-7.0$, $-2.9$]} \\
Reversal product $-$ batch & \multicolumn{2}{c}{$-0.29$ [$-0.43$, $-0.15$]} & \multicolumn{2}{c}{0.1 [$-0.9$, 1.2]} & \multicolumn{2}{c}{$-4.8$ [$-6.9$, $-3.0$]} \\
PriDe then reversal $-$ PriDe & \multicolumn{2}{c}{$-0.63$ [$-0.73$, $-0.53$]} & \multicolumn{2}{c}{$-4.3$ [$-5.0$, $-3.7$]} & \multicolumn{2}{c}{$-4.8$ [$-7.0$, $-2.9$]} \\
PriDe then reversal $-$ reversal & \multicolumn{2}{c}{$-0.18$ [$-0.37$, $-0.00$]} & \multicolumn{2}{c}{$-3.0$ [$-4.4$, $-1.8$]} & \multicolumn{2}{c}{0.0 [$-1.0$, 1.0]} \\
Batch then reversal $-$ batch & \multicolumn{2}{c}{$-0.66$ [$-0.75$, $-0.58$]} & \multicolumn{2}{c}{$-3.6$ [$-4.2$, $-3.0$]} & \multicolumn{2}{c}{$-4.8$ [$-6.8$, $-3.0$]} \\
Batch then reversal $-$ reversal & \multicolumn{2}{c}{$-0.37$ [$-0.49$, $-0.25$]} & \multicolumn{2}{c}{$-3.7$ [$-4.9$, $-2.6$]} & \multicolumn{2}{c}{$-0.1$ [$-1.0$, 0.9]} \\
PriDe then batch $-$ PriDe & \multicolumn{2}{c}{$-0.17$ [$-0.21$, $-0.14$]} & \multicolumn{2}{c}{$-1.9$ [$-2.3$, $-1.6$]} & \multicolumn{2}{c}{0.4 [$-0.4$, 1.4]} \\
PriDe then batch $-$ batch & \multicolumn{2}{c}{$-0.03$ [$-0.12$, 0.07]} & \multicolumn{2}{c}{$-0.5$ [$-1.3$, 0.2]} & \multicolumn{2}{c}{0.4 [$-0.7$, 1.5]} \\
\bottomrule
\end{tabular}}
\caption{\textbf{PriDe and the reversal product narrow placement range by similar amounts, but the product leaves fewer flips and rank reversals; a correction followed by the product lowers both factors and narrows the placement range most.} Cells with at least three options, on shared held-out items. Top: mean log-factor changes, placement range and rank reversals as levels, and changes in accuracy and stable-wrong share. Bottom: cellwise paired differences, the reversal product minus each correction and each composition minus each of its components.}
\label{tab:composition}
\end{table*}

The two operations also leave different harms of option order. The \emph{placement range}, how far accuracy
moves when every correct option is placed at one letter, equals the flip rate times the range of $h$
(\cref{eq:placement}). \emph{Rank reversals} are the share of checkpoint pairs on a benchmark whose accuracy
order changes across the cyclic orderings.

\paragraph{The operations leave different harms.}
Prior correction and the reversal product narrow the placement range by similar amounts, yet on the same
cells the product leaves fewer rank reversals than either correction (\cref{tab:composition}).
A pair can reverse only if correctness changes with the order, and a checkpoint's accuracies under two
orderings differ by at most its flip rate on the same items (\cref{sec:method_stats}); correction leaves the flip rate almost
unchanged and does not lower rank reversals. It narrows the placement range instead through the range of
$h$, which follows the letter of the correct option rather than of the chosen answers; a lower skew alone
need not narrow it. Lowering one factor need not reduce the harm the other leaves, so a method
should be judged on both factors and both harms.

\paragraph{Composition lowers both factors.}
A prior correction followed by the reversal product keeps the product's fall in the flip rate, lowers the skew
more than the product alone, and narrows the placement range the most (\cref{tab:composition}). Against
each component on the same cells, it further lowers the other's factor, but rank reversals stay at
the product's level. \citet{zheng2024large} already pair PriDe with partial permutation averaging.
Two corrections in a row add little: PriDe followed by batch calibration, with the second prior re-estimated
on PriDe's output, does not clearly lower the pooled SD beyond batch calibration and barely moves the flip rate.
\looseness=-1

\paragraph{Stable is not the same as correct.}
The reversal product lowers the flip rate partly by settling items on a wrong option: the stable-wrong share
rises by more than accuracy does, and both compositions with the product do the same (\cref{tab:composition}).
Averaged across cells, for nine of ten newly stable-wrong items the option they settle on was already the model's unique most
frequent choice over the item's cyclic and reversed queries. Fewer order-dependent answers are not more right
answers, so accuracy belongs beside both factors.
\looseness=-1

\section{Conclusion}
\label{sec:conclusion}

Under balanced orderings, the pooled SD is exactly the product of two behaviours: the flip rate, how
often the order changes an answer, and the positional skew, how unevenly the changed answers land; balanced RStd
shares only the flip rate unless $K=2$. Mitigation methods move these behaviours differently and leave
different harms. Prior correction mainly lowers the skew and aggregating an ordering with its reversal mainly the flip
rate; the two narrow the placement range by similar amounts, yet on the same cells the aggregation leaves
fewer rank reversals and lowers flips partly by settling items on a wrong answer. A correction followed by the aggregation lowers
both factors. A lower factor alone is therefore not evidence of a better model. Flips also track accuracy and output
uncertainty, and the skew tracks the model's option-symbol bias. Both factors come from one sweep over the
orderings without gold labels; we recommend reporting them with accuracy and judging a mitigation method by
all three and by the harm of interest.
\looseness=-1
\label{sec:body_end}

\section*{Limitations}
\label{sec:limitations}
\label{sec:limitations_score}

Our analysis characterises the observable effects of option order, but not the model-internal mechanisms that produce them. The prior corrections and controlled interventions act on output scores, and the observed associations with accuracy and output uncertainty do not establish causal explanations. Moreover, although the pooling identity is exact under balanced orderings, reducing either factor does not guarantee better predictions: consistently wrong answers can remain stable, and changes in positional skew need not translate into improved placement robustness or ranking stability. Our experiments also focus on deterministic readouts of mostly English multiple-choice benchmarks and specific mitigation configurations. Future work should investigate the internal mechanisms underlying the two factors through targeted interventions and theoretically characterise when reductions in each factor lead to improvements in correctness and evaluation stability.

\bibliography{custom,added}

\clearpage
\appendix
\raggedbottom
\etocdepthtag.toc{mtappendix}
\etocsettagdepth{mtbody}{none}
\etocsettagdepth{mtappendix}{subsection}
\etocsettocstyle{\section*{Appendix contents}}{}
\tableofcontents
\let\obtAppendixSection\section
\renewcommand{\section}{\FloatBarrier\obtAppendixSection}
\setcounter{topnumber}{4}\setcounter{bottomnumber}{2}\setcounter{totalnumber}{6}\setcounter{dbltopnumber}{3}
\renewcommand{\topfraction}{0.9}\renewcommand{\bottomfraction}{0.6}\renewcommand{\textfraction}{0.1}
\renewcommand{\floatpagefraction}{0.75}\renewcommand{\dbltopfraction}{0.9}\renewcommand{\dblfloatpagefraction}{0.75}
\section{Experimental protocols and cohorts}
\label{app:pools}

\subsection{Cohorts, models and benchmarks}
\label{app:pools_rule}

An item is usable only if it was answered under every ordering. A cell enters a filtered pool if it
holds at least $90\%$ of its benchmark's modal usable item count under the first-token sweep.
Cells that fail the filter remain in the unfiltered pool, which is every cell run. Each readout has its
own pool. The first-token pool is every cell run (\cref{tab:benchmarks}). Chain of thought was run with the plain
prompt on every benchmark but MMLU-Pro, five-shot prompting on every benchmark but TruthfulQA-MC1, BBH-MC and
MMLU-Pro, and direct answers on ARC, MMLU, OpenBookQA and LogiQA, with HellaSwag on two checkpoints. The filter
removes no first-token or five-shot cell. The held-out benchmark control uses the six benchmarks run on every
checkpoint, and the all-ordering analyses use eight checkpoints on ARC-Easy, ARC-Challenge, HellaSwag and MMLU,
six of them for the item partition. Analyses on part of the first-token pool name it: the reversal product, its
compositions and the canonical-margin controls use the cells with at least three options, since reversing a
two-option ordering gives the other cyclic ordering, and the common-readout comparison uses the cells run under the
first-token, chain-of-thought and direct-answer readouts, before the filter, with five-shot on the part of them it covers.

\paragraph{Models.}
\label{app:pools_models}
\Cref{tab:models} lists the roster.
The DeepSeek-R1-Distill checkpoints \citep{guo2025deepseek_r1} are fine-tuned on DeepSeek-R1
traces from base backbones, three of which are in the roster, and stay unpaired with them, because
distillation from a reasoning teacher varies a different axis from the post-training the rest of the
roster varies.
The scale trends within a family in \cref{tab:reversals} group checkpoints by model name, which puts
DeepSeek-R1-Distill-Llama-8B with the Llama checkpoints and the Qwen-based distilled checkpoints in a group
of their own.

\begin{table*}[t]
\centering
\small
\setlength{\tabcolsep}{3pt}
\begin{tabular}{@{}lr@{\hspace{2em}}lr@{}}
\toprule
\textbf{Model} & \textbf{Params (B)} & \textbf{Model} & \textbf{Params (B)} \\
\midrule
\multicolumn{2}{@{}l}{\textit{Llama-3.x}} & \multicolumn{2}{@{}l}{\textit{Qwen3}} \\
Llama-3.2-1B & 1 & Qwen3-0.6B & 0.6 \\
Llama-3.2-1B-Instruct & 1 & Qwen3-1.7B & 1.7 \\
Llama-3.2-3B & 3 & Qwen3-4B & 4 \\
Llama-3.2-3B-Instruct & 3 & Qwen3-8B & 8 \\
Llama-3.1-8B & 8 & \multicolumn{2}{@{}l}{\textit{Gemma-2}} \\
Llama-3.1-8B-Instruct & 8 & gemma-2-2b & 2 \\
Llama-3.1-70B & 70 & gemma-2-2b-it & 2 \\
Llama-3.1-70B-Instruct & 70 & gemma-2-9b & 9 \\
\multicolumn{2}{@{}l}{\textit{Qwen2.5}} & gemma-2-9b-it & 9 \\
Qwen2.5-0.5B & 0.5 & gemma-2-27b & 27 \\
Qwen2.5-0.5B-Instruct & 0.5 & gemma-2-27b-it & 27 \\
Qwen2.5-1.5B & 1.5 & \multicolumn{2}{@{}l}{\textit{DeepSeek-R1-Distill}} \\
Qwen2.5-1.5B-Instruct & 1.5 & DeepSeek-R1-Distill-Qwen-1.5B & 1.5 \\
Qwen2.5-3B & 3 & DeepSeek-R1-Distill-Qwen-7B & 7 \\
Qwen2.5-3B-Instruct & 3 & DeepSeek-R1-Distill-Llama-8B & 8 \\
Qwen2.5-7B & 7 & DeepSeek-R1-Distill-Qwen-14B & 14 \\
Qwen2.5-7B-Instruct & 7 & DeepSeek-R1-Distill-Qwen-32B & 32 \\
Qwen2.5-14B & 14 & \multicolumn{2}{@{}l}{\textit{Mistral}} \\
Qwen2.5-14B-Instruct & 14 & Mistral-Nemo-Instruct-2407 & 12 \\
Qwen2.5-32B & 32 & Mistral-Small-24B-Instruct-2501 & 24 \\
Qwen2.5-32B-Instruct & 32 & \multicolumn{2}{@{}l}{\textit{Phi-3}} \\
Qwen2.5-72B-Instruct & 72 & Phi-3-mini-4k-instruct & 3.8 \\
 &  & Phi-3.5-mini-instruct & 3.8 \\
 &  & Phi-3-medium-4k-instruct & 14 \\
\bottomrule
\end{tabular}
\caption{\textbf{The checkpoints span seven families.} Parameter counts are the released sizes. The Qwen3 releases carry no \texttt{-Base} suffix and are post-trained hybrid-thinking checkpoints.}
\label{tab:models}
\end{table*}

\begin{wraptable}{R}{0.48\textwidth}
\setlength{\columnwidth}{\linewidth}
\centering
\small
\begin{tabular}{lrrr}
\toprule
\textbf{Benchmark} & \textbf{$K$} & \textbf{Items} & \textbf{Models} \\
\midrule
PIQA & 2 & 1,838 & 41 \\
WinoGrande & 2 & 1,267 & 41 \\
ARC-Challenge & 4 & 1,165 & 41 \\
ARC-Easy & 4 & 2,365 & 41 \\
BBH-MC & 4 & 732 & 36 \\
HellaSwag & 4 & 10,042 & 41 \\
LogiQA & 4 & 651 & 36 \\
MMLU & 4 & 535 & 41 \\
OpenBookQA & 4 & 500 & 36 \\
RACE-middle & 4 & 1,436 & 36 \\
TruthfulQA-MC1 & 4 & 219 & 36 \\
MMLU-Pro & 10 & 9,981 & 20 \\
\bottomrule
\end{tabular}
\caption{\textbf{Every statistic is computed on the items of these benchmarks that are usable under all $K$ orderings.} ``Items'' is the modal usable item count under the first-token sweep; ``Models'' counts checkpoints with a first-token cell. MMLU is a five-subject subset.}
\label{tab:benchmarks}
\end{wraptable}

\paragraph{Benchmarks.}
\label{app:pools_benchmarks}
The twelve benchmarks are PIQA \citep{bisk2020piqa}, WinoGrande \citep{sakaguchi2021winogrande},
ARC-Easy and ARC-Challenge \citep{clark2018think}, HellaSwag \citep{zellers2019hellaswag}, a
five-subject subset of MMLU \citep{hendrycks2021measuring}, OpenBookQA
\citep{mihaylov2018openbookqa}, LogiQA \citep{liu2020logiqa}, RACE-middle \citep{lai2017race},
TruthfulQA-MC1 \citep{lin2022truthfulqa}, BBH-MC \citep{suzgun2023challenging} and MMLU-Pro
\citep{wang2024mmlupro}.
Evaluation uses the test split where one is public and the validation split otherwise (PIQA,
WinoGrande, HellaSwag and TruthfulQA-MC1), and an item is kept only if its option list has exactly $K$ entries, so
TruthfulQA-MC1 and BBH-MC keep their four-choice items, most of BBH-MC's from temporal sequences, ruin names and
movie recommendation. A BBH-MC question shows only the presented options, without the source's lettered list.

MMLU is a $535$-item subset of five subjects: abstract algebra, anatomy, college chemistry, computer
security and global facts.

\subsection{Prompt}
\label{app:pools_prompt}

One fixed question--options body is used for every model, benchmark and ordering.

\begin{quote}\ttfamily\small
\noindent Question: \{question\}\\
A. \{option 0\}\\
B. \{option 1\}\\
\ldots\\
Answer:
\end{quote}

Letters name positions, so a rotation moves contents between letters. The open
checkpoints get no system prompt. For LogiQA and RACE-middle the passage precedes the question and the
options, the order under which the option tokens can attend to the passage \citep{ok2026lost}.

\paragraph{A flipping item.}
Phi-3-mini-4k-instruct saw this ARC-Easy test item in the canonical order as
\begin{quote}\ttfamily\small
\noindent Question: Which type of energy does a radiator give off ?\\
A. heat\\
B. light\\
C. solar\\
D. electrical\\
Answer:
\end{quote}
with the gold option at \opt{A}. Under $\pi_1$, $\pi_2$ and $\pi_3$ (\cref{eq:shift}) the gold option moves to \opt{B}, \opt{C}
and \opt{D}, and position \opt{A} holds \texttt{electrical}, \texttt{solar} and \texttt{light} in turn. The model answers \opt{A}
under all four orderings, so its chosen content changes and the item flips: it is correct in the canonical
order and wrong under the other three. Under those three orderings the gold option keeps 0.44, 0.44 and 0.32
of the option probability, while position \opt{A} takes 0.56, 0.56 and 0.68.

\subsection{Readouts}
\label{app:pools_readouts}

\paragraph{First-token probability.}
One forward pass per ordering, softmaxed over the $K$ option letters (\cref{sec:method_stats}), with
ties broken by letter index.
Probability analyses on the cyclic orderings use this softmax on every first-token cell or a stated
subset, such as the cells with at least three options for reversal and composition. The all-ordering
operation check uses the cells whose option probabilities are stored unrounded under every ordering
(\cref{app:orderings}).

\paragraph{Five-shot prompting.}
The first-token readout after five exemplars from the benchmark's development or training split, held
fixed across models and orderings.

\paragraph{Chain of thought.}
The final \texttt{Answer:} cue is replaced, after a blank line, by \texttt{"Let's think step by step, then provide your final answer as
\textbackslash"The answer is [X]\textbackslash" where X is A, B, C, or D."}, with \texttt{"A or B"} for two options. The answer is the latest valid explicit answer in the trace or,
when there is none, the last valid standalone letter. Only the first $\min(K,4)$ letters are valid,
and a trace with no valid answer leaves the item unusable for the cell. Dropping the cells in which more than half
of the extracted answers come from the standalone-letter fallback leaves the skew's share of the variance of the log
pooled SD compatible with the first-token readout's and its correlation with accuracy negative.
Chain-of-thought pools use the plain prompt without a checkpoint-specific chat template. No generative readout includes MMLU-Pro.

\paragraph{Direct answer.}
The base template. The choice is the letter the generated text opens with; a text that opens with no
valid letter is scored with the chain-of-thought extraction rule.

\subsection{Decoding}
\label{app:pools_decoding}

\begin{assumption}[Greedy readout of the open-weight checkpoints]
\label{asm:greedy}
For the open-weight checkpoints, the first-token and five-shot readouts take an option-probability argmax, and the generative readouts decode greedily
and extract answers deterministically. We assume that repeating a readout under the same ordering gives
the same answer, so that a flip between orderings reflects the reordering.
\end{assumption}
Sampled readouts mix order sensitivity with draw noise, so we do not use them. The closed models of
\cref{app:readouts_codex} can answer a repeated prompt differently, so they fall outside the assumption and are
analysed separately.
\looseness=-1

\paragraph{Compute.}
All experiments are inference only. They ran on NVIDIA GPUs with 24 to 141\,GB of memory (RTX 3090, RTX 4090, RTX A5000,
RTX A6000, RTX 6000 Ada, L40S, A100 and H200), one GPU per run and two for the largest checkpoints, with PyTorch 2.6,
Hugging Face Transformers 4.57 (4.45 for the Phi-3 checkpoints, whose remote code requires it) and vLLM 0.8.5 for
generation.

\subsection{Intervals}
\label{app:pools_intervals}
\label{app:matched_blocks}

The factor and mitigation analyses of the main text (the canonical margin, the held-out prior and its
permuted control, the operations and their compositions) use checkpoint-block bootstrap draws and nominal 95\%
percentile intervals; their contrasts are paired within each draw and are not adjusted for multiple
comparisons.
The remaining analyses keep their own resampling units, named with each table. Correlations across cells
resample cells, paired so that all statistics compared share each draw. Correlations under a held-out design,
with accuracy scored on items, benchmarks or orderings the statistic does not read, resample whole
checkpoints instead, and leave-one-out refits are reported as the range over refits. Control contrasts use
checkpoint-block intervals, which resample whole checkpoints with all their cells. Most intervals are central
95\% percentile intervals; the variance shares of the log pooled SD and the scaled-flip-rate correlations use
the 1.25th and 98.75th percentiles, central 97.5\% intervals that are wider to allow for the several quantities
each of these analyses reports at once, and gap intervals marked simultaneous hold jointly over the gaps of
one analysis.
The controls resample checkpoints while holding the retained items and benchmarks fixed, so their intervals do not account for item sampling.
A change between two measurements, of two cells or of one cell
before and after PriDe, counts as resolved when its item-bootstrap interval excludes zero.

\setcounter{topnumber}{4}
\setcounter{dbltopnumber}{3}
\renewcommand{\topfraction}{0.9}
\renewcommand{\dbltopfraction}{0.9}
\renewcommand{\textfraction}{0.07}
\renewcommand{\floatpagefraction}{0.75}
\renewcommand{\dblfloatpagefraction}{0.75}
\section{Formal statements and the two factors}
\label{app:formal}

This appendix proves the statements the method relies on and collects the cell-level evidence for the two factors of the pooled selection-rate SD. Notation follows \cref{sec:method_stats}. The benchmark has $N \ge 1$ items, each with the same number $K \ge 2$ of options and one correct content index. For every item and ordering, a single chosen content is fixed, with ties broken by a fixed rule. Comparisons between ordering sets use the same items and the same fixed choices, so they differ only in which orderings are included. For a nonempty set $S$ of orderings we write $\FL(S)$ for the flip rate taken over $S$, so the cyclic protocol reports $\FL = \FL(S_{\mathrm{cyc}})$. Positions, content indices and summation indices are $0$-based. The method rests on one property of the cyclic orderings, which we state first.

\begin{lemma}[Balance]
\label{lem:balance}
Fix a content index $j$. Under \cref{eq:shift}, the map $s \mapsto \pi_s(j) = (j+s) \bmod K$ is a bijection of $\{0,\dots,K-1\}$. Hence over the $K$ cyclic orderings, content $j$ occupies each position exactly once, and for every position $r$ there is exactly one $s$ with $\pi_s(j) = r$.
\end{lemma}

\begin{proof}
    For fixed $j$, adding $j$ modulo $K$ is invertible, with inverse $r \mapsto (r - j) \bmod K$. 
\end{proof}

\subsection{The identity under any ordering design}
\label{app:any_design}

We extend the pooling decomposition and identity in \Cref{eq:pooling_mixture,eq:rstd_factor} to arbitrary ordering designs. A design is a multiset $S$ of $m \ge 1$ orderings of the $K$ options. For content index $j$, let
\begin{equation}
    u_j(r) = \frac{1}{m}\, \#\{\pi \in S : \pi(j)=r\},
\end{equation}
counting repeated orderings with multiplicity, and define $M(S)=\max_j\operatorname{std}(u_j)$, which depends only on the design. Here $\operatorname{std}$ denotes the population standard deviation over a vector's $K$ entries.

Let $f_q^S$ denote item $q$'s chosen-position histogram over $S$, and let $f = N^{-1}\sum_q f_q^S$. An item whose chosen content $z_q(\pi)$ is constant on $S$ has $f_q^S = u_{z_q}$. Here $\FL = \FL(S)$ is the share of items whose chosen content changes across $S$, and $g$ is the mean histogram of these flipping items.

To account for unbalanced designs, define the stable-item deviation from uniform as
\begin{equation}
    b_S = \frac{1}{N} \sum_{q:\,\mathrm{stable}} \left(u_{z_q}-\mathrm{unif}\right).
\end{equation}
When $\FL=0$, we set $g=\mathrm{unif}$. When $\FL=1$, the sum defining $b_S$ is empty and $b_S=0$.

\begin{proposition}[Pooling under any design]
\label{prop:any_design}
For every design and every prediction pattern:
\begin{enumerate}[nosep]
    \item \textbf{Pooling decomposition.} The pooled histogram satisfies
    \begin{equation}
        f = (1-\FL)\mathrm{unif} + \FL g + b_S.
        \label{eq:any_design_decomp}
    \end{equation}

    \item \textbf{Corrected identity.} The design-corrected pooled SD $\tilde s(f) = \operatorname{std}(f-b_S)$ satisfies
    \begin{equation}
        \tilde s(f) = \FL\cdot\operatorname{std}(g),
    \end{equation}
    and can be computed without gold labels.

    \item \textbf{Deviation bound.} The deviation from the pooling identity
    is bounded by
    \begin{equation}
    \begin{aligned}
        &\bigl|\operatorname{std}(f) -\FL\cdot\operatorname{std}(g)\bigr|\\ &\quad\le \operatorname{std}(b_S)\\ &\quad\le (1-\FL)M(S).
    \end{aligned}
    \end{equation}
\end{enumerate}

Moreover, the original pooling identity $\operatorname{std}(f)=\FL\cdot\operatorname{std}(g)$ holds for every prediction pattern if and only if $S$ is position-balanced, i.e., $u_j=\mathrm{unif}$ for every $j$.
\end{proposition}

\begin{proof}
For the pooling decomposition, averaging the chosen-position histograms over stable and flipping items gives
\begin{equation}
    \begin{aligned}
    f &= \frac{1}{N}\sum_{q:\,\mathrm{stable}}u_{z_q} +\FL g\\
      &= (1-\FL)\mathrm{unif}+\FL g+b_S.
    \end{aligned}
\end{equation}

For the corrected identity, subtracting $b_S$ gives $f-b_S=(1-\FL)\mathrm{unif}+\FL g$. Adding a constant vector does not change the standard deviation and $\FL\ge0$, so $\operatorname{std}(f-b_S) =\FL\cdot\operatorname{std}(g)$. The correction uses only the ordering design and the chosen contents.

For the deviation bound, the standard deviation $\operatorname{std}(x) =\lVert x-\bar x\mathbf{1}\rVert_2/\sqrt{K}$, with $\bar x$ the mean entry of $x$ and $\mathbf{1}$ the all-ones vector, is a seminorm. Applying the reverse triangle inequality to \cref{eq:any_design_decomp} gives
\begin{equation}
    \bigl|\operatorname{std}(f) -\FL\cdot\operatorname{std}(g)\bigr| \le \operatorname{std}(b_S).
\end{equation}
By the triangle inequality and the definition of $M(S)$,
\begin{equation}
\begin{aligned} 
    \operatorname{std}(b_S) &\le \frac{1}{N} \sum_{q:\,\mathrm{stable}} \operatorname{std}(u_{z_q})\\ &\le (1-\FL)M(S).
\end{aligned}
\end{equation}

For the equivalence, a position-balanced design has $u_j=\mathrm{unif}$ for every $j$, so $b_S=0$ and \cref{eq:any_design_decomp} reduces to \cref{eq:pooling_mixture}. Conversely, if $u_j\ne\mathrm{unif}$ for some $j$, consider the prediction pattern in which every item always chooses $j$. Then $\FL=0$ and $f=u_j$, giving $\operatorname{std}(f)>0$ and violating the original pooling identity.
\looseness=-1
\end{proof}

When each item has its own nonempty design $S_q$, with $u_{q,j}$ computed from $S_q$ as above, the same statements hold with $u_{q,j}$ in place of $u_j$, $M=\max_{q,j}\operatorname{std}(u_{q,j})$ in place of $M(S)$ and $b$ in place of $b_S$, where
\begin{equation}
    b = \frac{1}{N} \sum_{q:\,\mathrm{stable}} \left(u_{q,z_q}-\mathrm{unif}\right).
\end{equation}

Item histograms are equally weighted even if design sizes differ. For the original pooling identity to hold for every prediction pattern, each $S_q$ must be position-balanced. When every item is stable, the identity requires $f=\mathrm{unif}$. Varying only one item's chosen content implies that $u_{q,j}$ is independent of $j$. Since $K^{-1}\sum_j u_{q,j}=\mathrm{unif}$, each $S_q$ must be balanced. Sufficiency follows from \cref{eq:any_design_decomp}.

The corrected residual is at numerical precision and the deviation bound holds in the cells run under all orderings and in every draw of each random-subset design (\cref{tab:design_correction}).

\subsection{Gold-label statistics and two-order disagreement}
\label{app:formal_gold}

\begin{proposition}[Placement accuracy and balanced RStd]
\label{prop:placement}
Under the cyclic orderings, for each letter $j$ let $s_q(j) = (j - c_q) \bmod K$ be the unique ordering (\cref{lem:balance}) that
shows item $q$'s correct option at letter $j$, and let $A_j = \frac{1}{N}\sum_q a_q(\pi_{s_q(j)})$. With
$\mathrm{SC}$ the share of items correct under every cyclic ordering, define
$h_j = \frac{1}{N\FL}\sum_{q \text{ flipping}} a_q(\pi_{s_q(j)})$
when $\FL>0$, and set $h_j=0$ for every $j$ when $\FL=0$. Then
\begin{equation}
  A_j = \mathrm{SC} + \FL \cdot h_j \quad\text{for every } j,
  \label{eq:placement_proof}
\end{equation}
so $\operatorname{std}_j(A_j) = \FL \cdot \operatorname{std}(h)$ and
$\max_j A_j - \min_j A_j = \FL\,(\max_j h_j - \min_j h_j)$.
\end{proposition}

\begin{proof}
With one correct content per item, an item correct under every cyclic ordering always chooses that content
and therefore does not flip. A stable item has constant correctness, since its chosen content does not change across orderings. 
So the stable-correct items contribute $\mathrm{SC}$ to every $A_j$ and the stable-wrong items contribute nothing.
If $\FL>0$, the remaining contribution is $\frac{1}{N}\sum_{q \text{ flipping}}a_q(\pi_{s_q(j)})=\FL\,h_j$.
If $\FL=0$, there are no remaining items and $A_j=\mathrm{SC}$ for every $j$, consistent with $h=0$.
Population SD and range are translation invariant and positively homogeneous, giving both identities.
\end{proof}

For a balanced multiset $S$ of $m$ orderings, replace $a_q(\pi_{s_q(j)})$ by
$a_{q,j}^S=\frac{K}{m}\sum_{\pi\in S:\,\pi(c_q)=j}a_q(\pi)$ and define $\FL$ and $\mathrm{SC}$ on $S$. Every item
is shown exactly $m/K$ times with its correct option at each letter, so the same proof gives all three identities.

$A_j$ is the classwise recall of letter $j$ when every item is evaluated once with its correct option at each
letter, so $\operatorname{std}_j(A_j)$ is RStd \citep{zheng2024large} with the correct option balanced over the letters; RStd
computed on the canonical ordering alone does not split this way. $\max_j A_j - \min_j A_j$ is the cyclic form
of the answer-moving attack. Unlike $g$, computing $h$ requires gold labels. In general neither vector determines the other, although
$h=g$ under the cyclic protocol when $K=2$ and $\FL>0$. Across cells, the dispersions
$\operatorname{std}(h)$ and $\operatorname{std}(g)$ have a Spearman correlation of $0.92$.

\begin{proposition}[Two-order disagreement]
\label{prop:two_order}
For orderings $\pi, \pi' \in S$, the share of items with $z_q(\pi) \neq z_q(\pi')$ is at most $\FL(S)$.
\end{proposition}

\begin{proof}
An item answered differently under $\pi$ and $\pi'$ has a non-constant choice on $S$.
\end{proof}

The bound requires both orderings to belong to $S$. For $K\ge3$, reversal is not a cyclic ordering, so an item
can keep its choice under every cyclic ordering and change it under reversal. If $S=\{\pi,\pi'\}$ with two
distinct orderings, their disagreement rate equals $\FL(S)$; this includes the cyclic protocol at $K=2$. When
$S$ contains additional distinct orderings, a fixed pair's disagreement rate and $\FL(S)$ need not determine
one another: an item can change only outside the pair, and a flipping item can either agree or disagree within
the pair. Neither scalar reports the destinations of the changing answers.

\subsection{The divergence bound}
\label{app:formal_bound}

On a balanced design $f - u = \FL\,(g - u)$ with $u$ uniform, so every deviation from uniform that is
positively homogeneous of degree one in $f - u$ factors like \cref{eq:rstd_factor}: total variation and the largest
deviation from uniform each equal $\FL$ times their value for $g$, and the chi-square divergence,
homogeneous of degree two, equals $\FL^2\,\chi^2(g \,\|\, u)$.
The KL divergence lacks that homogeneity. We use base-two logarithms (bits), with
$0\log_2 0 = 0$. Writing
$f = (1 - \FL)\,u + \FL\,g$, convexity of $D_{\mathrm{KL}}(\cdot \,\|\, u)$ in its first argument gives
\begin{equation}
  \begin{split}
    D_{\mathrm{KL}}(f \,\|\, u) &\;\le\; (1 - \FL)\,D_{\mathrm{KL}}(u \,\|\, u) \\
    &\quad + \FL\,D_{\mathrm{KL}}(g \,\|\, u) \\
    &\;=\; \FL \cdot D_{\mathrm{KL}}(g \,\|\, u),
  \end{split}
  \label{eq:kl_bound}
\end{equation}
with equality exactly when $\FL \in \{0, 1\}$ or $g = u$, by strict convexity in the first argument
(the uniform reference has full support). The bound holds in every first-token cell.

\subsection{The skew and the placement factor}
\label{app:formal_skew_placement}

The placement range is $\FL\cdot\operatorname{range}(h)$, and $h$, the accuracy of flipping items by gold
letter, is not a function of $g$. With two options every flipping item picks one fixed position, so $h = g$ when $\FL > 0$,
balanced RStd is $\FL\cdot\operatorname{std}(g)$ and the placement range is $2\FL\cdot\operatorname{std}(g)$; at a
fixed flip rate a lower skew lowers both. With more options the two can come apart:
a four-option item whose chosen contents over the cyclic orderings are $(0,1,0,1)$ has $\FL = 1$ and
$g = (\tfrac12, 0, \tfrac12, 0)$, and its placement range is $1$ if the correct content is $0$ but $0$ if it is
$2$. The same flip rate and skew can therefore carry the largest or no placement harm. On the four-option cells,
scored on the operation evaluation items, the mean cosine between the centred $g$ and $h$ is 0.92, so the link is close in practice but not an identity.

With the predictions held fixed, $\FL > 0$ and the correct content of every item drawn uniformly from its options, the
expected $h$ equals $g$: the orderings that place a content at
letter $j$ are those under which a correct answer of that content lands at $j$, so averaging $h_j$ over the
possible correct contents counts how often the answers land at $j$. The difference $h - g$ thus measures how
the actual correct contents line up with the letters the model favours. On the four-option cells its standard deviation across
letters is 0.04, against 0.12 for $h$. Changes move together as well: across the cells with at least three options, between the reversal product and
PriDe followed by it, the changes in skew and in $\operatorname{range}(h)$ have a Spearman correlation of 0.92 across cells, yet
in 3\% of cells the skew falls while the placement range rises, and in 6\% under PriDe against no
correction. \Cref{tab:composition} therefore reports the
placement range directly.

\section{Factor diagnostics: margins, prior and accuracy controls, closed models}
\label{app:factor_support}

This section gives the analyses behind \cref{sec:factors}: the canonical-margin and prior results with
their intervals, how the two factors share the variation of the pooled SD and how reliably they are
measured, a second prior estimator, and the controls on the flip rate's association with accuracy, and the two
factors on three closed models.
\looseness=-1

\subsection{Canonical margins and the option-symbol prior}
\label{app:factor_diag}

\begin{table*}[t]
\centering
\small
\setlength{\tabcolsep}{4pt}
\begin{tabular}{@{}lr@{}}
\toprule
Quantity & Estimate [95\% CI] \\
\midrule
\multicolumn{2}{l}{\textit{Canonical-order AUC, flipping against stable items}} \\
Top-two gap, all cells & $0.848\,[0.828, 0.863]$ \\
Top-two gap, least accurate third & $0.781\,[0.755, 0.806]$ \\
Top-two gap, middle accurate third & $0.832\,[0.813, 0.853]$ \\
Top-two gap, most accurate third & $0.908\,[0.898, 0.925]$ \\
Entropy, all cells & $0.848\,[0.825, 0.862]$ \\
\midrule
\multicolumn{2}{l}{\textit{Canonical-order AUC, flipping against one stable state}} \\
Top-two gap, against stable-correct & $0.877\,[0.854, 0.890]$ \\
Top-two gap, against stable-wrong & $0.697\,[0.687, 0.708]$ \\
\midrule
\multicolumn{2}{l}{\textit{Option-symbol prior}} \\
Partial Spearman with skew, given accuracy & $0.499\,[0.396, 0.589]$ \\
Direction cosine, $K=4$, estimated prior & $0.869\,[0.824, 0.899]$ \\
Direction cosine, $K=4$, letter-permuted & $-0.097\,[-0.111, -0.083]$ \\
Direction cosine, $K=10$, estimated prior & $0.881\,[0.834, 0.910]$ \\
Direction cosine, $K=10$, letter-permuted & $-0.028\,[-0.046, -0.002]$ \\
\bottomrule
\end{tabular}
\caption{\textbf{The canonical top-two gap separates flipping items in every accuracy third, and flipping answers point the way the estimated prior does.} Median AUC (a smaller gap or larger entropy scores as flipping), the partial Spearman between the skew and the standard deviation of the estimated prior after adjusting both for accuracy rank, and the median cosine between $g$ and the prior, each minus uniform; the prior is estimated on the held-out 5\% of each cell's items. Each AUC summarises only cells containing both compared item states. For the letter-permuted prior each cell takes the median cosine over its non-identity letter permutations, which places it near zero.}
\label{tab:factor_diagnostics}
\end{table*}

\begin{table*}[t]
\centering
\small
\setlength{\tabcolsep}{4pt}
\begin{tabular}{@{}lr@{}}
\toprule
Quantity & Estimate [95\% CI] \\
\midrule
\multicolumn{2}{l}{\textit{Strength $\alpha=0.5$}} \\
$\Delta\log$ skew, estimated & $-0.316\,[-0.397, -0.238]$ \\
$\Delta\log$ skew, permuted & $0.132\,[0.103, 0.165]$ \\
$\Delta\log$ skew, difference & $-0.449\,[-0.528, -0.373]$ \\
$\Delta\log$ flip rate, estimated & $-0.025\,[-0.032, -0.018]$ \\
$\Delta\log$ flip rate, permuted & $0.027\,[0.021, 0.033]$ \\
$\Delta\log$ flip rate, difference & $-0.052\,[-0.062, -0.043]$ \\
$\Delta$ accuracy (pp), estimated & $0.398\,[0.314, 0.484]$ \\
$\Delta$ accuracy (pp), permuted & $-0.171\,[-0.222, -0.124]$ \\
$\Delta$ accuracy (pp), difference & $0.569\,[0.463, 0.683]$ \\
\midrule
\multicolumn{2}{l}{\textit{Strength $\alpha=1$}} \\
$\Delta\log$ skew, estimated & $-0.453\,[-0.582, -0.324]$ \\
$\Delta\log$ skew, permuted & $0.296\,[0.249, 0.346]$ \\
$\Delta\log$ skew, difference & $-0.750\,[-0.867, -0.638]$ \\
$\Delta\log$ flip rate, estimated & $-0.021\,[-0.034, -0.008]$ \\
$\Delta\log$ flip rate, permuted & $0.071\,[0.057, 0.087]$ \\
$\Delta\log$ flip rate, difference & $-0.092\,[-0.110, -0.074]$ \\
$\Delta$ accuracy (pp), estimated & $0.469\,[0.352, 0.594]$ \\
$\Delta$ accuracy (pp), permuted & $-0.575\,[-0.692, -0.466]$ \\
$\Delta$ accuracy (pp), difference & $1.045\,[0.854, 1.247]$ \\
\bottomrule
\end{tabular}
\caption{\textbf{Dividing out the estimated prior lowers the skew and raises accuracy; the same prior with permuted letters does the opposite.} Mean change on the items not used to estimate the prior, after dividing by the prior raised to $\alpha$; permuted rows average over letter permutations, and differences are paired within each draw of the 95\% checkpoint-block intervals. For each factor and strength, log-change rows use the common cells where that factor is positive before correction and after correction by the estimated prior and every included letter permutation. The flip rate moves little either way.}
\label{tab:factor_interventions}
\end{table*}

\begin{wrapfigure}{R}{0.48\textwidth}
\setlength{\columnwidth}{\linewidth}
\centering
\includegraphics[width=\linewidth]{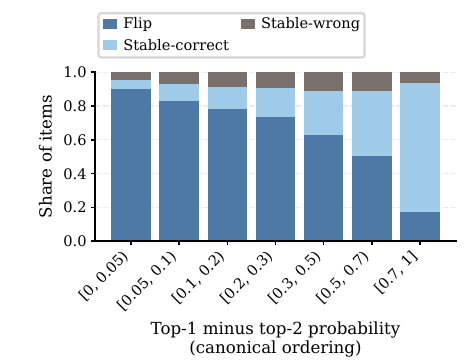}
\caption{\textbf{Items flip where their two leading options are close.} Share of stable-correct, flipping and stable-wrong items by the top-two letter-probability gap under the canonical order on all first-token cells, cells weighted equally.}
\label{fig:factors_gap}
\end{wrapfigure}

\Cref{tab:factor_diagnostics} gives the canonical-order and prior results of \cref{sec:factors} with their
intervals, and \cref{tab:factor_interventions} the prior correction at both tested strengths. The top-two gap
separates flipping from stable items (\cref{fig:factors_gap}) in every accuracy third, most sharply in the most accurate third, and
separates them less sharply from stable-wrong than from stable-correct items; entropy under the same ordering
does about as well. At
half strength the estimated prior still lowers the skew and the permuted one still raises it, by smaller amounts.

\begin{table}[t]
\centering
\small
\setlength{\tabcolsep}{4pt}
\begin{tabularx}{\linewidth}{@{}>{\raggedright\arraybackslash}Xr@{}}
\toprule
Flip outcome & Median AUC [95\% CI] \\
\midrule
All $K$ answers & $0.865\,[0.844, 0.883]$ \\
Canonical answer left out & $0.819\,[0.790, 0.838]$ \\
Left out; canonical answer correct & $0.837\,[0.814, 0.858]$ \\
Left out; canonical answer wrong & $0.633\,[0.617, 0.645]$ \\
All $K$ answers; items wrong under every ordering & $0.760\,[0.750, 0.771]$ \\
Left out; stratified by own and peer accuracy & $0.752\,[0.738, 0.772]$ \\
\bottomrule
\end{tabularx}
\caption{\textbf{The canonical top-two gap predicts flips on the other orderings and among items wrong everywhere.} Median AUC over first-token cells with at least three options and both outcome classes for the stated comparison. \emph{Left out} counts an item as flipping only when its other $K-1$ answers disagree, so the answer the gap is read from is not in the outcome. The last row compares items only within strata of their own canonical correctness and of five bins of the other checkpoints' accuracy on the item, keeping items seen by at least three of them.}
\label{tab:margin_controls}
\end{table}

Two features of the canonical-margin result could make it circular: the gap is read from the same canonical
answer that enters the flip outcome, and accuracy separates flipping from stable items by itself.
\Cref{tab:margin_controls} removes each in turn on the cells with at least three options, as in the mitigation
comparison. Leaving the canonical answer out of the outcome lowers the AUC but keeps it well above
chance, more so when the canonical answer is correct, and comparing items only within strata of their own
correctness and of how often other checkpoints answer them correctly keeps most of it. Among items wrong under
every ordering, whose accuracy is fixed at zero, the gap still separates the items whose wrong answer moves
from those that keep one wrong answer. The gap thus measures how closely the options compete, which accuracy
alone does not.
\looseness=-1

\subsection{Variation and reliability of the factors}
\label{sec:results_factors}
\label{sec:mechanism}
\label{sec:mechanism_partition}
\label{sec:mechanism_pooling}

\begin{table*}[t]
\centering
\small
\setlength{\tabcolsep}{4pt}
\begin{tabular}{@{}lcccc@{}}
\toprule
 & \textbf{All cells} & \textbf{$K{=}2$} & \textbf{$K{=}4$} & \textbf{$K{=}10$} \\
\midrule
\multicolumn{5}{@{}l}{\textit{Share of the variance of the log pooled selection-rate SD}} \\
Flip rate, $\log\FL$ & \begin{tabular}[t]{@{}c@{}}$0.41$\\{\scriptsize$[0.35, 0.48]$}\end{tabular} & \begin{tabular}[t]{@{}c@{}}$0.22$\\{\scriptsize$[0.14, 0.31]$}\end{tabular} & \begin{tabular}[t]{@{}c@{}}$0.45$\\{\scriptsize$[0.39, 0.52]$}\end{tabular} & \begin{tabular}[t]{@{}c@{}}$0.07$\\{\scriptsize$[0.04, 0.14]$}\end{tabular} \\
Skew factor, $\log\operatorname{std}(g)$ & \begin{tabular}[t]{@{}c@{}}$0.35$\\{\scriptsize$[0.30, 0.40]$}\end{tabular} & \begin{tabular}[t]{@{}c@{}}$0.62$\\{\scriptsize$[0.48, 0.79]$}\end{tabular} & \begin{tabular}[t]{@{}c@{}}$0.23$\\{\scriptsize$[0.20, 0.28]$}\end{tabular} & \begin{tabular}[t]{@{}c@{}}$0.64$\\{\scriptsize$[0.53, 0.74]$}\end{tabular} \\
Twice their covariance & \begin{tabular}[t]{@{}c@{}}$0.24$\\{\scriptsize$[0.18, 0.29]$}\end{tabular} & \begin{tabular}[t]{@{}c@{}}$0.16$\\{\scriptsize$[0.02, 0.28]$}\end{tabular} & \begin{tabular}[t]{@{}c@{}}$0.31$\\{\scriptsize$[0.26, 0.36]$}\end{tabular} & \begin{tabular}[t]{@{}c@{}}$0.29$\\{\scriptsize$[0.19, 0.36]$}\end{tabular} \\
\bottomrule
\end{tabular}

\vspace{6pt}

\setlength{\tabcolsep}{3pt}
\begin{tabular}{@{}lcccc@{}}
\toprule
 & \multicolumn{2}{c}{\textit{Own pool}} & \multicolumn{2}{c}{\textit{Cells common to the readouts}} \\
\textbf{Readout} & \textbf{\begin{tabular}[b]{@{}c@{}}Positional\\skew share\end{tabular}} & \textbf{\begin{tabular}[b]{@{}c@{}}Skew, Spearman\\with accuracy\end{tabular}} & \textbf{\begin{tabular}[b]{@{}c@{}}Positional\\skew share\end{tabular}} & \textbf{\begin{tabular}[b]{@{}c@{}}Skew, Spearman\\with accuracy\end{tabular}} \\
\midrule
First-token probability & \begin{tabular}[t]{@{}c@{}}$0.35$\\{\scriptsize$[0.30, 0.40]$}\end{tabular} & \begin{tabular}[t]{@{}c@{}}$-0.333$\\{\scriptsize$[-0.417, -0.237]$}\end{tabular} & \begin{tabular}[t]{@{}c@{}}$0.21$\\{\scriptsize$[0.17, 0.26]$}\end{tabular} & \begin{tabular}[t]{@{}c@{}}$-0.399$\\{\scriptsize$[-0.564, -0.210]$}\end{tabular} \\
Five-shot & \begin{tabular}[t]{@{}c@{}}$0.34$\\{\scriptsize$[0.28, 0.40]$}\end{tabular} & \begin{tabular}[t]{@{}c@{}}$-0.365$\\{\scriptsize$[-0.472, -0.255]$}\end{tabular} & \begin{tabular}[t]{@{}c@{}}$0.21$\\{\scriptsize$[0.17, 0.25]$}\end{tabular} & \begin{tabular}[t]{@{}c@{}}$-0.421$\\{\scriptsize$[-0.602, -0.219]$}\end{tabular} \\
Chain of thought & \begin{tabular}[t]{@{}c@{}}$0.45$\\{\scriptsize$[0.35, 0.59]$}\end{tabular} & \begin{tabular}[t]{@{}c@{}}$-0.223$\\{\scriptsize$[-0.408, -0.036]$}\end{tabular} & \begin{tabular}[t]{@{}c@{}}$0.30$\\{\scriptsize$[0.20, 0.40]$}\end{tabular} & \begin{tabular}[t]{@{}c@{}}$-0.423$\\{\scriptsize$[-0.625, -0.158]$}\end{tabular} \\
Direct answer, leading letter & \begin{tabular}[t]{@{}c@{}}$0.19$\\{\scriptsize$[0.14, 0.25]$}\end{tabular} & \begin{tabular}[t]{@{}c@{}}$-0.531$\\{\scriptsize$[-0.671, -0.349]$}\end{tabular} & \begin{tabular}[t]{@{}c@{}}$0.21$\\{\scriptsize$[0.16, 0.28]$}\end{tabular} & \begin{tabular}[t]{@{}c@{}}$-0.461$\\{\scriptsize$[-0.606, -0.301]$}\end{tabular} \\
\bottomrule
\end{tabular}
\caption{\textbf{The flip rate and the skew carry comparable shares of the variance of the log pooled SD.} Top: over the first-token cells, shares carried by the log flip rate, the log skew and twice their covariance, which sum to one. Bottom: the skew's share and its Spearman correlation with accuracy for each readout, on its own pool and on the cells run under the first-token, chain-of-thought and direct-answer readouts, with checkpoint-block intervals. The five-shot row uses the subset of the common pool with a five-shot readout. The skew is negatively associated with accuracy in every readout and pool.}
\label{tab:decomposition}
\end{table*}

\begin{table*}[t]
\centering
\small
\setlength{\tabcolsep}{4pt}
\begin{tabular}{@{}lccc@{}}
\toprule
 & \textbf{Spearman with accuracy} & \textbf{Split-half Spearman} & \textbf{Spearman--Brown} \\
\midrule
Flip rate $\FL$ & \begin{tabular}[t]{@{}c@{}}$-0.947$\\{\scriptsize$[-0.957, -0.934]$}\end{tabular} & $0.993$ & $0.996$ \\
Positional skew, $\operatorname{std}(g)$ & \begin{tabular}[t]{@{}c@{}}$-0.333$\\{\scriptsize$[-0.403, -0.254]$}\end{tabular} & $0.913$ & $0.954$ \\
Positional skew, median within a benchmark & \begin{tabular}[t]{@{}c@{}}$-0.389$\\{\scriptsize$[-0.526, -0.281]$}\end{tabular} & $0.947$ & $0.973$ \\
\bottomrule
\end{tabular}
\caption{\textbf{The flip rate and the skew are reliable across random item halves.} Over the first-token cells: Spearman correlations with accuracy, with 95\% cell-resampling intervals, and split-half correlations, uncorrected and Spearman--Brown corrected; the last row gives the median over benchmarks of each within-benchmark value, its interval resampling cells within each benchmark.}
\label{tab:three_factor}
\end{table*}

Across cells the flip rate is strongly associated with accuracy ($\spear = -0.947$) and the positional skew
$\operatorname{std}(g)$ more weakly ($\spear = -0.333$), so the skew can reorder the cells that the flip rate
ranks (\cref{fig:decomposition}).
On cells with positive pooled SD the identity gives $\log\operatorname{std}(f)=\log\FL+\log\operatorname{std}(g)$,
so the variance of the log pooled SD across cells is the sum of the two log-factor variances and twice their
covariance; we use natural logs and population moments over the same cells. Across all cells the flip-rate
term takes the largest share of this variance, the skew term slightly less and the covariance term the rest; the flip rate's
term is larger on the four-option cells and the skew's on the two- and ten-option cells, which come from
different benchmarks (\cref{tab:decomposition}).

Recomputed on two random halves of
each cell's items, both factors rank the cells nearly the same way (\cref{tab:three_factor}).
A conservative test also rejects uniform position choice in nearly two fifths of the first-token cells after
Holm correction. For each cell it bounds the largest deviation of the pooled position shares from uniform with a
weighted Hoeffding union bound that treats each item's $K$ answers as one cluster and the items as
independent draws within the checkpoint and benchmark.

\subsection{A second prior estimator}
\label{sec:results_skew_prior}
\label{app:skew_prior}

\begin{table*}[t]
\centering
\small
\begin{tabular}{llrl}
\toprule
\textbf{Comparison} & \textbf{Cells} & \textbf{Estimate} & \textbf{95\% interval} \\
\midrule
Prior and skew, Spearman & All & $0.745$ & $[0.680, 0.799]$ \\
Prior and skew, given accuracy & All & $0.728$ & $[0.664, 0.782]$ \\
Prior and skew, given accuracy & $K=2$ & $0.763$ & $[0.632, 0.848]$ \\
Prior and skew, given accuracy & $K=4$ & $0.677$ & $[0.601, 0.742]$ \\
Prior and skew, given accuracy & $K=10$ & $0.623$ & $[0.092, 0.838]$ \\
\bottomrule
\end{tabular}
\caption{\textbf{The mean option-symbol profile tracks the skew on disjoint items.} Spearman correlations across first-token cells between the profile's nonuniformity and the skew; \emph{given accuracy} rows are partial rank correlations, overall and by option count.}
\label{tab:skew_prior}
\end{table*}

\Cref{sec:factors} estimates the option-symbol prior with PriDe. As a second estimate we use a simpler one: on the
held-out 5\% of a cell's
items, the arithmetic mean over items and cyclic orderings of the probabilities the checkpoint gives the
option letters, whose nonuniformity is its SD across the $K$ letters. PriDe instead averages normalised
geometric means over orderings (\cref{app:pride}). On the other items we compute the skew of the flipping
items, so the two quantities share no item. Their nonuniformities rank the cells closely, also after removing
the linear association of each rank with the rank of the cell's cyclic accuracy and within each option count,
with a wide interval for the ten-option cells (\cref{tab:skew_prior}).

\begin{figure}[t]
\centering
\includegraphics[width=\linewidth]{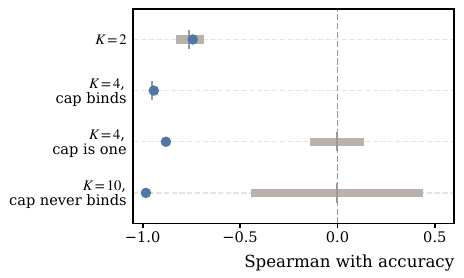}
\caption{\textbf{The flip rate's association with accuracy exceeds the cap null where the cap does not bind.} First-token cells by option count and cap stratum: dots are observed Spearman correlations with accuracy, grey bars the central 95\% of a null that permutes the flip rate divided by its cap $\min(1, K(1-A))$, and ticks its median.}
\label{fig:controls}
\label{fig:controls-cap}
\end{figure}

\subsection{Controls on the flip rate's association with accuracy}
\label{sec:results_controls}
\label{app:robustness}
\label{app:robustness_null}
\label{app:robustness_controls}

\begin{table*}[t]
\centering
\small
\begin{tabular}{lrrrr}
\toprule
\textbf{Cap null} & \textbf{Observed $\spear$} & \textbf{Null median} & \textbf{Null 95\% range} & \textbf{$P$} \\
\midrule
$K{=}2$, all & $-0.743$ & $-0.764$ & $[-0.830, -0.685]$ & $0.71$ \\
\quad binding & $-0.744$ & $-0.760$ & $[-0.830, -0.682]$ & $0.67$ \\
\quad cap equal to one & $1.000$ & $-1.000$ & $[-1.000, 1.000]$ & $1.00$ \\
$K{=}4$, all & $-0.968$ & $-0.623$ & $[-0.675, -0.569]$ & $<0.001$ \\
\quad binding & $-0.943$ & $-0.951$ & $[-0.960, -0.941]$ & $0.95$ \\
\quad cap equal to one & $-0.881$ & $0.000$ & $[-0.138, 0.135]$ & $<0.001$ \\
$K{=}10$, all & $-0.983$ & $0.000$ & $[-0.445, 0.438]$ & $<0.001$ \\
Cap not binding, all $K$ & $-0.884$ & $-0.051$ & $[-0.178, 0.072]$ & $<0.001$ \\
\bottomrule
\end{tabular}
\par\smallskip
\begin{tabular}{lr}
\toprule
\textbf{Scaled flip rate, cap binding} & \textbf{$\spear$ with accuracy} \\
\midrule
All option counts & $-0.226$ $[-0.379, -0.081]$ \\
$K{=}2$ & $0.149$ $[-0.027, 0.347]$ \\
$K{=}4$ & $0.042$ $[-0.151, 0.224]$ \\
\bottomrule
\end{tabular}
\caption{\textbf{The cap null is compatible with the flip rate's association with accuracy where the cap binds, but not on four-option cells where the cap equals one.} The scaled flip rate is $\FL/\min(1,K(1-A))$, and the null permutes it within each option count, and also within each cap subset for the binding and cap-one rows; the cap binds where $\min(1,K(1-A))<1$, and the two-option row where it equals one holds only two cells, where any Spearman correlation is $\pm1$. Null ranges are 2.5th--97.5th percentiles, and $P$ is the share of permutations, with a plus-one correction, whose $|\spear|$ reaches the observed one; lower-block intervals resample checkpoints (1.25th--98.75th).}
\label{tab:nulls}
\end{table*}

\begin{table*}[t]
\centering
\small
\setlength{\tabcolsep}{4pt}
\begin{tabular}{lrrr}
\toprule
Statistic & Item-disjoint & Benchmark-disjoint & Ordering-disjoint \\
\midrule
Flip rate $\FL$ & $-0.941$ & $-0.892$ & $-0.980$ \\
 & $[-0.949, -0.930]$ & $[-0.931, -0.811]$ & $[-0.993, -0.935]$ \\
Per-item KL to uniform & $-0.927$ & $-0.865$ & $-0.977$ \\
 & $[-0.940, -0.909]$ & $[-0.917, -0.764]$ & $[-0.992, -0.931]$ \\
Per-item selection-rate SD & $-0.855$ & $-0.873$ & $-0.977$ \\
 & $[-0.879, -0.816]$ & $[-0.921, -0.777]$ & $[-0.992, -0.931]$ \\
Pooled KL to uniform & $-0.808$ & $-0.688$ & $-0.939$ \\
 & $[-0.840, -0.771]$ & $[-0.769, -0.563]$ & $[-0.979, -0.881]$ \\
Pooled selection-rate SD & $-0.757$ & $-0.689$ & $-0.935$ \\
 & $[-0.795, -0.712]$ & $[-0.770, -0.562]$ & $[-0.978, -0.875]$ \\
Other count-based statistics & $-0.924$ to $-0.935$ & $-0.875$ to $-0.880$ & $-0.964$ to $-0.979$ \\
\midrule
Gap, divergence & $+0.119$ & $+0.204$ & $+0.038$ \\
 & $[0.089, 0.153]$ & $[0.132, 0.289]$ & $[-0.005, 0.077]$ \\
Gap, dispersion & $+0.097$ & $+0.203$ & $+0.042$ \\
 & $[0.063, 0.132]$ & $[0.130, 0.291]$ & $[0.000, 0.076]$ \\
\bottomrule
\end{tabular}
\caption{\textbf{The flip rate stays negatively associated with accuracy measured on disjoint items or orderings.} Signed Spearman correlations and magnitude gaps (per-item minus pooled; flip rate minus pooled for benchmark-disjoint), on first-token cells scored on held-out item folds, six benchmarks each scored against the other five, and all-ordering cells scored on held-out orderings; benchmark-disjoint estimates average within-benchmark correlations, and gap intervals are simultaneous over each column's two gaps. The other count-based statistics are the fluctuation rate, the mean pairwise disagreement between orderings and the mean entropy of each item's chosen content, given as the range of their correlations. In the all-ordering cells the per-item gaps are smaller and their intervals reach zero; scoring those cells on their own cyclic orderings instead leaves the gaps unchanged.}
\label{tab:fl_controls}
\end{table*}

\begin{table*}[t]
\centering
\small
\setlength{\tabcolsep}{5pt}
\begin{tabular}{@{}lrrrrrr@{}}
\toprule
\textbf{Model} & \textbf{Accuracy} & $\FL$ & $\operatorname{std}(g)$ & \textbf{Pooled SD} & \textbf{Repeat disagreement} & \textbf{Reorder disagreement} \\
\midrule
\texttt{gpt-6.1-sol} & $0.890$ & $0.036$ & $0.019$ & $0.0007$ & $0.017$ & $0.016$ \\
\texttt{gpt-6-luna} & $0.736$ & $0.342$ & $0.011$ & $0.0038$ & $0.168$ & $0.210$ \\
\texttt{gpt-5.6-terra} & $0.867$ & $0.116$ & $0.026$ & $0.0030$ & $0.062$ & $0.067$ \\
\bottomrule
\end{tabular}
\caption{\textbf{Reordering changes a closed model's answer about as often as asking the same ordering again.} MMLU, four cyclic orderings, one-letter reply. Both disagreement columns compare two calls on each item: repeat disagreement pairs ordering 0 with a second call on it, and reorder disagreement pairs ordering 0 with each other ordering and averages the three. Only for \texttt{gpt-6-luna} does reordering change answers more often than repeating, by an amount whose item-bootstrap interval excludes zero.}
\label{tab:codex_readout}
\end{table*}

Although the flip rate uses no gold label, it is computed from the same predictions as accuracy, and accuracy caps
it (\cref{eq:fl_cap}). The cap alone can produce an association, since an accurate cell has little room to
flip. Dividing a cell's flip rate by its cap, $\min(1, K(1-A))$, gives the \emph{scaled flip rate}; the cap
null permutes scaled flip rates across the cells of one option count, and within one cap
subset for the analyses restricted to binding or cap-one cells, which holds each cell's feasible region fixed. This null is compatible with the observed association where the cap binds, almost
everywhere on the two-option cells (\cref{fig:controls-cap}). With four options the cap equals one whenever
accuracy is at most 0.75, so on those cells it imposes no restriction beyond $\FL \le 1$; there the null
centres on zero, yet the flip rate still falls closely with accuracy ($\spear = -0.881$; \cref{tab:nulls}).
These results rule out the cap as a complete explanation, though correctness could still act through other routes.

Separately, we measure accuracy on items, benchmarks or orderings that the flip rate does not use. With accuracy read off one of five balanced
item folds and the flip rate off the other four, rotated over all five, the flip rate tracks it about as
closely as on all items. With each of the six benchmarks run on every checkpoint set against mean accuracy
on the other five, and on the cells run under all $K!$ orderings with accuracy read off the non-cyclic
orderings, it stays strongly associated with accuracy (\cref{tab:fl_controls}). Each control resamples whole
checkpoints.

\subsection{Closed models}
\label{app:readouts_codex}

We read three closed models on the MMLU items under the four cyclic orderings
with a one-letter reply, and ask ordering 0 a second time. The same prompt can return a different answer, so
the flip rate mixes order sensitivity with repeat disagreement. The flip rate compares four calls and a repeat
compares two, so we compare like with like: two calls on the same ordering against two calls on different
orderings. Reordering changes answers about as often as repeating for two of the three models and somewhat
more often for the third (\cref{tab:codex_readout}). Repeat disagreement is therefore substantial next to reorder disagreement,
and these measurements do not identify how much of the flip rate is due to reordering, so we do not set their factors beside the open-weight cells. The conservative uniformity test
above detects a positional skew in none of the three.

\section{Operations and composition}
\label{app:operations}
\label{app:pools_operations}

This section gives the protocol of the mitigation operations in \cref{sec:operations,sec:composition}, the
decomposition of the letter preference that the reversal product averages, and each operation's effects.

\subsection{Protocol}
\label{app:operations_protocol}

All operations use the same first-token predictions, items and orderings. Each comparison uses a calibration split of 5\% of each cell's items, at least one, and evaluates every
operation on the remaining 95\%; no gold label enters any prior or operation. The cohort is the cells with at
least three options, each scored under its cyclic orderings and their reversals.
Each item is presented in its $K$ cyclic orderings; in our protocol an operation may query further orderings provided no two
presented orderings share a query, and every statistic is computed over the $K$ answers, each counted at the letter that holds the chosen content in its presented ordering. Comparisons of the corrections alone use
every first-token cell, two-option cells included. The prior
experiment of \cref{sec:factors} uses the same split.

\paragraph{Prior correction.}
\label{app:pride}
PriDe's prior is estimated on the calibration items: for each of them
the normalised geometric mean over the $K$ orderings of its option-letter probabilities, averaged over
those items \citep{zheng2024large}. It is divided out of every ordering of the evaluated items; when a cell is scored on its canonical ordering alone the prior is applied
to that ordering only.
The published method instead
scores the estimation items by permutation averaging and the remaining items in their default ordering, and
averages over five splits. Batch calibration divides by the mean letter probability over the estimation
items' orderings, an arithmetic-mean variant of the log-score correction of \citet{zhou2024batch}, which
estimates over the test batch; variants that average log probabilities, or estimate over all test items,
move both factors the same way.
Neither correction adds an option, and the argmax is over the original $K$.

\paragraph{Aggregation.} The probability product multiplies, for each content, its probabilities under the
presented ordering and under that ordering's reversal and returns the content with the largest product,
following the forward and backward pairing of \citet{wei-etal-2024-unveiling}; exact ties go to the earliest
letter of the presented ordering. The reversal of a cyclic ordering is not cyclic for $K \ge 3$, so the partners of different
presented orderings are distinct. A majority vote over a fixed pool of orderings
\citep{pezeshkpour2024large}, with ties broken by content rather than by presented position, returns one
content for every presented ordering; its flip rate, pooled SD and
balanced RStd are zero by construction, and with no flipping item its skew is undefined; the convention
$g=\mathrm{unif}$ at $\FL=0$ in \cref{app:formal} is algebraic, not an estimate of that skew. We therefore treat the vote as
the zero-flip case and compare prior correction with the reversal product on the factors.

\paragraph{Compositions and controls.} A composition applies a correction and then an aggregation to the
corrected probabilities; composing two corrections re-estimates the second prior on the first one's output.
The letter-permuted control divides by the estimated prior after permuting its letters, every non-identity
permutation for four options, the swap for two and distinct non-identity permutations drawn uniformly at random
for ten, and averages the result. The canonical margin is
the gap between the two largest letter probabilities under $\pi_0$ alone. Rank reversals count the
checkpoint pairs on a benchmark whose accuracy difference changes sign across the presented cyclic orderings, on the
items shared by all of them; counting each pair on the items its two checkpoints share instead lowers every
level, and the composition's difference from the reversal product stays unresolved under both counts. A composition is compared with each of its components on the same cells, with
paired checkpoint-block intervals.

\subsection{What the reversal product cancels}
\label{app:operations_cancel}

Let a letter preference $b$ add $b_i$ to the log score of whatever content letter $i$ shows, with letters indexed
$i = 0, \ldots, K-1$. Reversal maps letter $i$ to $K-1-i$, and the product of an
ordering and its reversal adds both log scores of a content, so a content shown at letter $i$ receives
$b_i + b_{K-1-i}$ on top of the scores the model gives it without $b$. Writing $b = b^{\mathrm{sym}} + b^{\mathrm{anti}}$ with
$b^{\mathrm{sym}}_i = (b_i + b_{K-1-i})/2$, the summed score receives $2b^{\mathrm{sym}}$; halving it leaves the argmax
unchanged, so the product keeps $b^{\mathrm{sym}}$, a preference that treats
mirrored letters alike, and cancels $b^{\mathrm{anti}}$. For four options every centred preference is
$\alpha\,(1,1,-1,-1) + \beta\,(1,-1,-1,1) + \gamma\,(1,-1,1,-1)$ over \opt{A} to \opt{D}; the product cancels the
front-to-back part $\alpha$ and the alternating part $\gamma$ and keeps the outer-to-inner part $\beta$. A
preference for \opt{A} alone, $(3,-1,-1,-1)$, keeps $(1,-1,-1,1)$; a steady drift $(3,1,-1,-3)$ cancels entirely.
The remaining skew also depends on the rest of the scores and on which items still flip, so the identity
does not fix the change in skew. Empirically, estimating each cell's letter preference
from the mean letter logits of its held-out 5\% of items, the share that survives the mirroring,
$\operatorname{std}(b^{\mathrm{sym}})/\operatorname{std}(b)$, has a Spearman correlation of 0.29 [0.17, 0.41] with the
reversal product's change in log skew on the other 95\%.

\subsection{Adding a known letter preference}
\label{app:operations_injection}

To test the decomposition on real scores, we add a letter-preference vector of known shape, with mean zero over
the letters and a root-mean-square of one logit, to every option logit of the MMLU cells, re-estimate the priors on the calibration items and apply every operation again
(\cref{fig:injection}). Under a front-to-back or an alternating preference, the two shapes that mirroring
cancels, the reversal product's answers do not change at all, so the uncorrected rises in the placement range
never reach it. An outer-versus-inner preference passes through the
product almost in full, and a preference for \opt{A} alone
passes through in its symmetric part. PriDe largely removes every added shape, and so does PriDe followed by the
product; batch calibration absorbs the outer-versus-inner shape only in part. The addition is an intervention on output scores, so it tests what each operation does
to a preference of a given shape; it does not show that models form their preferences additively, which the
association across cells in \cref{sec:operations} speaks to only indirectly.

\subsection{Effects of each operation}
\label{app:operations_effects}

\begin{table*}[t]
\centering
\small
\setlength{\tabcolsep}{4pt}
\begin{tabular}{@{}lrrrr@{}}
\toprule
Operation & \shortstack{Stop\\flipping (\%)} & \shortstack{Start\\flipping (\%)} & \shortstack{Resolved flips\\settling wrong (\%)} & \shortstack{Former flips on\\favoured letter (\%)} \\
\midrule
PriDe & $3.5\,[2.9, 4.0]$ & $2.7\,[2.3, 3.1]$ & $38.1\,[35.1, 40.9]$ & $22.8\,[21.6, 23.8]$ \\
Batch & $2.6\,[2.2, 3.1]$ & $1.3\,[1.1, 1.6]$ & $39.1\,[36.3, 41.9]$ & $27.5\,[26.5, 28.5]$ \\
PriDe then batch & $3.7\,[3.1, 4.2]$ & $2.5\,[2.3, 2.8]$ & $39.4\,[36.8, 42.0]$ & $22.8\,[21.8, 23.6]$ \\
Reversal product & $15.9\,[14.7, 17.0]$ & $0.9\,[0.8, 1.0]$ & $38.8\,[36.3, 41.3]$ & $26.4\,[25.5, 27.3]$ \\
PriDe then reversal & $16.7\,[15.3, 18.1]$ & $1.0\,[0.9, 1.2]$ & $39.0\,[36.5, 41.5]$ & $22.5\,[21.9, 23.0]$ \\
Batch then reversal & $16.9\,[15.6, 18.3]$ & $0.9\,[0.8, 1.0]$ & $39.2\,[36.6, 41.7]$ & $24.4\,[23.9, 24.9]$ \\
\bottomrule
\end{tabular}
\caption{\textbf{Corrections move flips both ways, while the reversal product mostly resolves them, about two in five onto a wrong option.} Shares of evaluated items that stop or start flipping, the share of resolved flips that become stable-wrong, and the share of previously flipping items' answers on the PriDe prior's top letter ($34.9\%$ before any operation), on the cells with at least three options.}
\label{tab:operations_detail}
\end{table*}

\Cref{tab:operations_detail} gives each operation's effects on all eligible cells. Under the corrections items
stop and start flipping, so the net flip rate moves by about a point while answers leave the favoured letter;
the reversal product mostly resolves flips, about two in five of them onto a wrong option.

\begin{figure*}[t]
\begin{minipage}[t]{0.48\textwidth}\vspace{0pt}
\setlength{\columnwidth}{\linewidth}
\centering
\small
\setlength{\tabcolsep}{4pt}
\begin{tabularx}{\linewidth}{@{}>{\raggedright\arraybackslash}Xr@{}}
\toprule
Property of the settled wrong content & Mean [95\% CI] \\
\midrule
Sole most-chosen content (\%) & $90.4\,[89.3, 91.3]$ \\
Most-chosen content, ties included (\%) & $98.9\,[98.7, 99.1]$ \\
Share of cyclic answers choosing it (\%) & $63.8\,[63.1, 64.5]$ \\
Among the top two in every query (\%) & $65.7\,[63.0, 68.2]$ \\
Item already wrong under every cyclic ordering (\%) & $37.9\,[36.3, 39.5]$ \\
Canonical top-two gap & $0.420\,[0.363, 0.477]$ \\
Canonical top-two gap, newly stable-correct & $0.484\,[0.426, 0.541]$ \\
Canonical top-two gap, stable-wrong before & $0.597\,[0.540, 0.653]$ \\
\bottomrule
\end{tabularx}
\captionof{table}{\textbf{When the reversal product makes a flipping item stably wrong, it usually settles on the option its queries already chose most often.} Items that flip without aggregation and become stable-wrong under the product, described by the wrong content they settle on, counted over the item's cyclic and reversed queries except where a row says cyclic. Shares and mean probability gaps within a cell, averaged over cells; the last two rows give the gap for comparison groups.}
\label{tab:stable_wrong_support}
\end{minipage}\hfill
\begin{minipage}[t]{0.48\textwidth}\vspace{0pt}
\setlength{\columnwidth}{\linewidth}
\centering
\includegraphics[width=\linewidth]{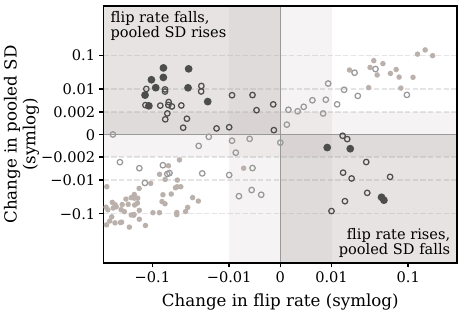}
\captionof{figure}{\textbf{A lower pooled SD need not mean fewer flips.} Base and post-trained checkpoint pairs on one benchmark, first-token readout. Dark grey points move the pooled SD and the flip rate in opposite directions and light grey points the same way, with shaded quadrants marking opposite moves; filled points have both item-bootstrap intervals excluding zero. A rise in the skew can outweigh a fall in the flip rate.}
\label{fig:reversals}
\label{fig:reversals-b}
\end{minipage}
\end{figure*}

\Cref{tab:stable_wrong_support} describes the wrong option such an item settles on, counting the item's answers
over all its cyclic and reversed queries. In nine of ten cases it is the single option chosen most
often, and in almost all others it ties for most often, so ties among equally frequent options decide few of
these items. The
settled option is among the top two in every query for about two in three items, and two in five
items were already wrong under every cyclic ordering before aggregation and only moved among wrong options. Their
canonical gaps are smaller than those of items the product settles correctly and much smaller than those of
items stable-wrong from the start. Among the items that become stable-wrong, aggregation therefore mostly settles
on a distractor the orderings already favoured. The items are selected by their outcome, so this describes the transition and does not identify its cause.
Shares are computed within each cell and averaged over cells; weighting cells by the share of their items whose flips
resolve instead raises the share of resolved flips that end wrong only slightly.

\subsection{PriDe on all cells and checkpoint comparisons}
\label{app:operations_checkpoints}
\label{sec:results_decision}
\label{app:reversals}

\begin{table*}[t]
\centering
\small
\setlength{\tabcolsep}{6pt}
\begin{tabular}{@{}lrrr@{}}
\toprule
 & \textbf{\begin{tabular}[b]{@{}c@{}}Pooled SD\\opposite (\%)\end{tabular}} & \textbf{\begin{tabular}[b]{@{}c@{}}Per-item KL\\opposite (\%)\end{tabular}} & \textbf{\begin{tabular}[b]{@{}c@{}}Opposite if\\resolved (\%)\end{tabular}} \\
\midrule
Post-training pairs & 30.1 & 6.8 & 15.9 \\
Checkpoint pairs & 17.3 & 4.6 & 9.6 \\
Scale trends & 5.8 & 0.0 & -- \\
PriDe cells & 20.0 & 17.7 & 10.5 \\
\bottomrule
\end{tabular}
\caption{\textbf{A comparison read off the pooled SD can point opposite to the flip rate, most often between a base checkpoint and its post-trained release.} Share of comparisons across first-token cells in which the pooled SD, or the per-item KL, moves opposite to the flip rate; scale trends compare rank correlations with parameter count within a family, training status and benchmark. The last column keeps comparisons with both changes resolved.}
\label{tab:reversals}
\end{table*}

As a secondary analysis on every cell, we apply PriDe with its prior
estimated on 5\% of each cell's items and compare both factors on the remaining items before and after. We
condition on the cells whose pooled SD falls, so the medians below describe attribution within that
selected set.

Among the cells where PriDe lowers the pooled SD, the skew's share of the log reduction,
$\Delta\log\operatorname{std}(g)/\Delta\log\operatorname{std}(f)$, has a median of 0.96. The flip rate's share is
one minus it, and either can fall outside zero and one when the two factors move in opposite directions. In about a quarter of those cells the flip rate's point estimate does not fall.
These results concern the pooled selection-rate SD, whereas \citet{zheng2024large} report RStd and accuracy, and accuracy with the answer moved to each letter. PriDe's reduction in the pooled SD is driven mainly by the skew of the remaining flips.
Where PriDe raises the pooled SD instead, the flip rate's share of the log increase has a median
under a tenth, so the change again runs mainly through the skew.

A comparison read off the pooled SD need not be a comparison of how often answers move. For a base
checkpoint and its post-trained release on one benchmark, the pooled SD and the flip rate point
opposite ways in about one pair in six, counting only pairs in which both changes are resolved by an
item bootstrap.
Post-trained releases usually flip less, and where the pooled SD rises regardless, the rise in the
skew of their remaining flips more than offsets the fall in the flip rate in their product
(\cref{fig:reversals-b}).
Trends with scale within a family rarely reverse (\cref{tab:reversals}).

Each comparison in \cref{tab:reversals} is read from the pooled selection-rate SD and from the flip
rate, and counts as opposite when the two changes differ in sign. A change in a post-training pair,
checkpoint pair or PriDe cell is resolved when its 95\% item-bootstrap interval excludes zero. Opposite directions
remain among comparisons with both changes resolved, at a lower rate.
Among post-training pairs, counting unresolved changes roughly doubles the share of opposite
directions, mostly through changes in the pooled SD much smaller than the typical one.

\obtAppendixSection{Cyclic orderings against all \texorpdfstring{$K!$}{K!}}
\label{app:orderings}

The cyclic protocol uses $K$ of the $K!$ orderings. Since any item that flips under the cyclic orderings also flips under the full set, the cyclic flip rate is a lower bound on the flip rate over all orderings. This appendix measures the gap, compares mitigation operations across balanced ordering sets, and tests other designs.

\begin{table*}[t]
\centering
\small
\setlength{\tabcolsep}{3pt}
\begin{tabular}{ll}
\toprule
\textbf{Condition} & \textbf{Measurement} \\
\midrule
\multicolumn{2}{l}{\textit{Number of orderings: cyclic against all $K!$}} \\
\quad $\FL_{\mathrm{cyc}} \le \FL_{\mathrm{full}}$ holds & $100\%$ of cells \\
\quad Relative undercount, mean over cells & $25.7\%$ \\
\quad $\spear(\mathrm{acc}, \FL)$, cyclic vs.\ all $K!$ & $-0.980$ vs.\ $-0.974$ \\
\midrule
\multicolumn{2}{l}{\textit{Which orderings: mean flip rate, and $\spear$ of the per-cell flip rate with the cyclic one}} \\
\quad Cyclic, $K$ orderings & $0.377$ \\
\quad Random, same size & $0.357$; $\spear = +0.999$ \\
\quad All $K!$ orderings & $0.472$; $\spear = 0.997$ \\
\midrule
\multicolumn{2}{l}{\textit{Cyclic flip rate minus its exact expectation over four-ordering subsets}} \\
\quad Mean; minimum & $0.021$; $0.001$ \\
\bottomrule
\end{tabular}
\caption{\textbf{The cyclic protocol undercounts flips in every cell but barely changes the ranking of cells.} Cells run under all 24 orderings of four options; random subsets average draws of four orderings, and the exact expectation averages over every four-ordering subset.}
\label{tab:scope}
\end{table*}

\subsection{How loose the bound is}
\label{app:orderings_bound}

The flip-rate bound holds in every cell evaluated under all $K!$ orderings, and the cyclic protocol understates the
flip rate by $0.095$ on average, yet the cyclic and full flip rates order the cells almost
identically and track accuracy almost equally (\cref{tab:scope}).

The cyclic orderings, which place the gold option once in each
position, give a flip rate $7.2\%$ higher per cell on average than a random subset of the same size,
with nearly the same order of cells.
The cyclic flip rate also exceeds its exact expectation over all four-ordering subsets in every cell.

\paragraph{Coverage of correctness-changing items.}
An item whose correctness is the same under all $K!$ orderings keeps it on any subset of them, so the cyclic
subset errs only by missing items whose correctness changes with the ordering; pooled over the cells of the six checkpoints used for the partition, it finds
$76.9\%$ of them.
\looseness=-1

\subsection{Operations beyond the cyclic orderings}
\label{app:orderings_operations}

\begin{table*}[t]
\centering
\small
\setlength{\tabcolsep}{4pt}
\begin{tabular}{@{}lrrrrr@{}}
\toprule
 & \multicolumn{2}{c}{Flip rate (\%)} & \multicolumn{2}{c}{Placement range (pp)} & \shortstack{Cyclic-stable items\\flipping elsewhere (\%)} \\
\cmidrule(lr){2-3}\cmidrule(lr){4-5}
 & Cyclic set & Other sets & Cyclic set & Other sets & \\
\midrule
\multicolumn{6}{@{}l}{\textit{All cells, uncorrected answers}} \\
Level & $37.7$ & $38.2$ ($37.5$, $38.7$) & $14.8$ & $14.8$ ($14.6$, $15.2$) & $10.5\,[6.9, 14.7]$ \\
\midrule
\multicolumn{6}{@{}l}{\textit{Cells with unrounded option probabilities}} \\
Uncorrected answers, level & $45.2$ & $45.4$ ($44.7$, $45.9$) & $26.1$ & $26.3$ ($25.8$, $26.6$) & $14.6$ \\
PriDe, change & $-0.1$ & $0.1$ ($-0.3$, $0.4$) & $-18.3$ & $-16.6$ ($-17.5$, $-15.7$) & $15.8$ \\
Reversal product, change & $-9.7$ & $-10.5$ ($-11.3$, $-9.3$) & $-11.6$ & $-11.7$ ($-11.9$, $-11.4$) & $4.6$ \\
PriDe then reversal, change & $-9.1$ & $-9.5$ ($-10.3$, $-8.4$) & $-21.7$ & $-21.7$ ($-22.2$, $-21.2$) & $6.3$ \\
\bottomrule
\end{tabular}
\caption{\textbf{The cyclic set is a typical balanced set of four orderings, and on the cells with unrounded probabilities each operation's effect is consistent across them.} Each set shows every content once at every position and holds no ordering with its reversal. Other-set columns give the mean and range over the other such sets; the last column is the share of cyclic-stable items that flip on another set, averaged over the other sets. The top row covers all cells run under all orderings, with a 95\% checkpoint-block interval for that share; the lower rows cover the all-ordering cells with unrounded option probabilities, using only items with unambiguous answers for each operation at the probability precision used and comparing each operation with the uncorrected answers on those same items.}
\label{tab:full_domain}
\end{table*}

The cyclic orderings are one of twelve sets of four orderings of four options that present every content once
at every position and contain no ordering together with its reversal. On each such set every method returns
four answers per item, and the reversal product pairs each presented ordering with a distinct partner, so
the sets compare the operations on equal terms (\cref{tab:full_domain}). On the cells scored under
all 24 orderings, the cyclic set is typical: its flip rate and placement range lie within the narrow spread of
the other sets. Stability does not transfer between sets, however: averaged over the other sets, about one item in ten that is stable
on the cyclic set flips on a given other set.

The operation rows use the all-ordering cells whose option probabilities are stored unrounded and, for each operation, only items whose answer under the operation is certified to stay the same when each stored probability varies
within a conservative six-decimal envelope. This keeps about nine in ten items under PriDe and the composition
and more under the reversal product.
On these items each
operation's effect is consistent across the sets: PriDe
leaves the flip rate and lowers the placement range, the reversal product lowers the flip rate by about
ten points, and the composition lowers both. Items stable on the cyclic set under the reversal product
flip on another set about a third as often as under the uncorrected answers.

\subsection{Other ordering designs}
\label{app:orderings_designs}

The pooling identity is exact on every position-balanced design tested, including every Latin-square set of four
orderings.
Random subsets of four orderings, drawn per item or shared by all items, leave the pooled SD about a sixth off the product at
the median, and further in the more accurate cells; balancing only the gold option's position leaves a
smaller gap (\cref{tab:designs}).
A random set shared by all items also removes the per-item SD's lead over the pooled one in tracking
accuracy.
Under all 24 orderings the skew carries a larger share of the pooled SD's variation than under the
cyclic orderings.

\begin{table*}[t]
\centering
\small
\setlength{\tabcolsep}{3pt}
\begin{tabular}{@{}llcccccc@{}}
\toprule
 & & \multicolumn{2}{c}{\textit{Pooling identity}} & \multicolumn{4}{c}{\textit{Spearman with accuracy}} \\
\textbf{Ordering design} & \textbf{\begin{tabular}[b]{@{}c@{}}Position\\balanced\end{tabular}} & \textbf{\begin{tabular}[b]{@{}c@{}}Relative\\residual,\\median\end{tabular}} & \textbf{\begin{tabular}[b]{@{}c@{}}Relative\\residual,\\largest\end{tabular}} & \textbf{\begin{tabular}[b]{@{}c@{}}Flip\\rate\end{tabular}} & \textbf{\begin{tabular}[b]{@{}c@{}}Positional\\skew\end{tabular}} & \textbf{\begin{tabular}[b]{@{}c@{}}Pooled\\SD\end{tabular}} & \textbf{\begin{tabular}[b]{@{}c@{}}Per-item\\minus\\pooled SD\\magnitude\end{tabular}} \\
\midrule
Cyclic & yes & $4\times10^{-16}$ & $6\times10^{-15}$ & \begin{tabular}[t]{@{}c@{}}$-0.980$\\{\scriptsize$[-0.99, -0.93]$}\end{tabular} & \begin{tabular}[t]{@{}c@{}}$-0.580$\\{\scriptsize$[-0.84, -0.36]$}\end{tabular} & \begin{tabular}[t]{@{}c@{}}$-0.935$\\{\scriptsize$[-0.98, -0.86]$}\end{tabular} & \begin{tabular}[t]{@{}c@{}}$0.042$\\{\scriptsize$[0.00, 0.08]$}\end{tabular} \\
All 24 orderings & yes & $4\times10^{-15}$ & $6\times10^{-14}$ & \begin{tabular}[t]{@{}c@{}}$-0.974$\\{\scriptsize$[-0.99, -0.92]$}\end{tabular} & \begin{tabular}[t]{@{}c@{}}$-0.621$\\{\scriptsize$[-0.86, -0.38]$}\end{tabular} & \begin{tabular}[t]{@{}c@{}}$-0.932$\\{\scriptsize$[-0.99, -0.86]$}\end{tabular} & \begin{tabular}[t]{@{}c@{}}$0.043$\\{\scriptsize$[-0.01, 0.08]$}\end{tabular} \\
Latin square & yes & $5\times10^{-16}$ & $3\times10^{-15}$ & $-0.979$ & $-0.584$ & $-0.924$ & $0.052$ \\
Gold at each position & gold only & $0.070$ & $0.186$ & $-0.977$ & $-0.577$ & $-0.922$ & $0.064$ \\
Random set per item & no & $0.157$ & $0.674$ & $-0.978$ & $-0.579$ & $-0.885$ & $0.072$ \\
Random set shared by items & no & $0.178$ & $0.694$ & $-0.976$ & $-0.599$ & $-0.872$ & $-0.250$ \\
\bottomrule
\end{tabular}
\caption{\textbf{The pooling identity is exact on every position-balanced ordering design and departs from it on unbalanced ones.} Rows are ways of choosing each item's orderings in the cells run under all 24 orderings; the gold row puts the correct option once at each position. The relative residual is $|\operatorname{std}(f) - \FL\cdot\operatorname{std}(g)|/\operatorname{std}(f)$; intervals resample checkpoints, and Latin-square and random rows are medians over sets and draws. The two residual columns take each cell's median and give its median and maximum over cells. The last column is the magnitude of the per-item SD's Spearman correlation with accuracy minus that of the pooled SD.}
\label{tab:designs}
\end{table*}

Subtracting the stable-item deviation $b_S$, which requires no gold labels, restores the pooling identity under every random draw, while the uncorrected deviation remains within its theoretical bound. In the median draw, the corrected statistic $\tilde s(f)$ tracks accuracy better than the uncorrected pooled SD and nearly as well as the cyclic design~(\cref{tab:design_correction}).

\begin{table*}[t]
\centering
\small
\setlength{\tabcolsep}{4pt}
\begin{tabular}{@{}lccccc@{}}
\toprule
\textbf{Design} & \begin{tabular}[b]{@{}c@{}}\textbf{Largest corrected}\\\textbf{absolute residual}\end{tabular} & \begin{tabular}[b]{@{}c@{}}\textbf{Bound}\\\textbf{violations}\end{tabular} & \begin{tabular}[b]{@{}c@{}}\textbf{Median residual}\\\textbf{/ bound}\end{tabular} & \begin{tabular}[b]{@{}c@{}}\textbf{Pooled SD}\\\textbf{Spearman}\end{tabular} & \begin{tabular}[b]{@{}c@{}}\textbf{Corrected SD}\\\textbf{Spearman}\end{tabular} \\
\midrule
Cyclic & $5.55\times10^{-17}$ & 0 & -- & $-0.935$ & $-0.935$ \\
\begin{tabular}[b]{@{}c@{}}Random set of 4 orderings\\shared by all items\end{tabular} & $1.39\times10^{-16}$ & 0 & $0.611$ & \begin{tabular}[t]{@{}c@{}}$-0.872$\\{\scriptsize$[-0.922, -0.791]$}\end{tabular} & \begin{tabular}[t]{@{}c@{}}$-0.924$\\{\scriptsize$[-0.947, -0.884]$}\end{tabular} \\
Random 4 orderings per item & $1.11\times10^{-16}$ & 0 & $0.734$ & \begin{tabular}[t]{@{}c@{}}$-0.885$\\{\scriptsize$[-0.926, -0.816]$}\end{tabular} & \begin{tabular}[t]{@{}c@{}}$-0.922$\\{\scriptsize$[-0.940, -0.889]$}\end{tabular} \\
\bottomrule
\end{tabular}
\caption{\textbf{Removing the design term restores the identity and, in the median draw, strengthens the accuracy association under random ordering subsets.} The corrected residual is $|\operatorname{std}(f - b_S)-\FL\operatorname{std}(g)|$, the largest over cells and draws; the bound $\operatorname{std}(b_S)$ is zero for the cyclic design, and the ratio column divides the uncorrected residual by it. Spearman correlations are with same-design accuracy; random designs show the median and {\scriptsize 2.5--97.5\%} range over draws.}
\label{tab:design_correction}
\end{table*}

\section{Metric definitions}
\label{app:metrics}
\label{app:metrics_defs}

Fix a model, a benchmark of $N$ items and the cyclic ordering set
$S_{\mathrm{cyc}} = \{\pi_0, \ldots, \pi_{K-1}\}$, and write $z_q(\pi_s)$ for the content index the
model returns under ordering $s$. Entropies are in bits. Every statistic in \cref{tab:metric_defs} is
computed from the same prediction matrix of $K$ forward passes and reads no gold label; the
fluctuation rate reads two of the passes.

\begin{table*}[t]
\centering
\small
\setlength{\tabcolsep}{4pt}
\begin{tabular}{@{}lll@{}}
\toprule
\textbf{Statistic} & \textbf{Definition} & \textbf{Related measure} \\
\midrule
\multicolumn{3}{@{}l}{\textit{Per-item consistency}} \\
Flip rate $\FL$
  & $\frac{1}{N}\bigl|\{q : z_q \text{ non-constant on } S_{\mathrm{cyc}}\}\bigr|$
  & Analogue of $1-{}$ReCon \citep{cocchieri-etal-2026-remedqa} \\
Fluctuation rate
  & $\frac{1}{N}\sum_q \mathbf{1}\bigl[z_q(\pi_0) \neq z_q(\pi_{\lfloor K/2 \rfloor})\bigr]$
  & adapted \citep{wei-etal-2024-unveiling} \\
\midrule
\multicolumn{3}{@{}l}{\textit{Pooled position statistics}} \\
$D_{\mathrm{KL}}$ to uniform
  & $\sum_{r} f(r) \log_2 \bigl(K f(r)\bigr)$
  & CKLD \citep{choi-etal-2025-mitigating} \\
Selection-rate SD
  & $\bigl(\frac{1}{K}\sum_{r} (f(r) - 1/K)^2\bigr)^{1/2}$
  & RStd \citep{zheng2024large} \\
\midrule
\multicolumn{3}{@{}l}{\textit{Per-item position statistics}} \\
Per-item $D_{\mathrm{KL}}$
  & $\frac{1}{N}\sum_q \sum_{r} f_q(r) \log_2 \bigl(K f_q(r)\bigr)$
  & -- \\
Per-item SD
  & $\frac{1}{N}\sum_q \operatorname{std}(f_q)$
  & -- \\
\bottomrule
\end{tabular}
\caption{\textbf{The flip rate and the fluctuation rate count content changes; position statistics measure positional imbalance, pooled or within each item.} All use the cyclic predictions: $z_q$ is the chosen content, $f(r)$ the share of all answers at position $r$ and $f_q(r)$ the share of item $q$'s answers there, and the fluctuation rate compares two cyclic orderings. The per-item statistics apply the pooled ones to each item's own answers and average over items. Larger values mean more disagreement or more imbalance.}
\label{tab:metric_defs}
\end{table*}

\paragraph{The fluctuation rate is adapted to the cyclic set.}
\citet{wei-etal-2024-unveiling} define it over a forward and a reversed ordering. For $K \ge 3$ no
rotation is the reversal, so we compare rotations $0$ and $\lfloor K/2 \rfloor$. For four options
these are ABCD and CDAB, where the published reversal is DCBA. At $K{=}2$ the definitions coincide.

\paragraph{The flip rate is taken over the cyclic set.}
ReCon \citep{cocchieri-etal-2026-remedqa} spans seven format variants; we count changes over the $K$ cyclic
orderings.

\paragraph{Label-free analogues of RStd and CKLD.}
RStd \citep{zheng2024large} is the standard deviation across classwise recalls, read under the premise
that gold answers are randomly positioned; our factorisation instead requires that every option of every
item visit every position equally often. Our selection-rate SD uses unconditional selection frequencies and
so reads no gold label. CKLD
\citep{choi-etal-2025-mitigating} computes $D_{\mathrm{KL}}(\mathrm{gold}\,\|\,\mathrm{predictions})$.
Our forward KL of selection frequencies to uniform, $D_{\mathrm{KL}}(f\,\|\,\mathrm{unif})$, substitutes uniform for gold and reverses the arguments.

\section{Closest related measurements and mitigations}
\label{app:related}

This section extends the comparisons of \cref{sec:related} to the closest measurements and mitigations.

\paragraph{Consistency measures for multiple choice.}
\citet{cocchieri-etal-2026-remedqa} define ReCon over seven input formats, including open answers mapped
back to the options. Its complement counts any inconsistency on that variant set, so it is analogous to,
but not numerically identical to, the cyclic flip rate. Their high multiple-choice scores with lower
cross-format consistency do not contradict the flip cap, which bounds the flip rate by accuracy averaged
over the same evaluated variants.
\citet{gupta2024changing} count an item as correct only when it is answered correctly under both the
original and a shuffled order. SCORE \citep{nalbandyan-etal-2025-score} reports a pairwise consistency
rate under reordered choices, paraphrased prompts and sampling seeds beside mean accuracy and accuracy ranges across those
perturbations, and finds consistency correlated with mean accuracy across models, as the flip rate and accuracy are in our
panel.
Consistency without correctness is also recognised. \citet{li-etal-2024-multiple} find that consistency
over repeated queries need not imply correctness, and \citet{cocchieri-etal-2026-remedqa} report ReCon
beside ReAcc, which also
requires every answer to be correct. Their repeated-query analysis is within items, whereas our
association with accuracy is across cells under cyclic reordering, and it likewise does not say whether an
individual stable answer is correct. We carry the distinction to mitigation, measuring whether the items
that stop flipping become stably correct or stably wrong.

\paragraph{Consistency and position preference are already reported together.}
Judge audits report position consistency, how often a verdict survives a swap, beside a direction score
built to shrink towards zero as consistency rises \citep{shi-etal-2025-judging}, or a swap-conflict rate
beside position-conditioned win rates \citep{wang-etal-2024-large-language-models-fair}. The raw fairness
score of \citet{shi-etal-2025-judging} already combines how often a pair is judged inconsistently with the
direction of the preference among those pairs. For two candidates without ties it follows from their
definition that its absolute value, divided by the number of swapped pairs, is twice the pooled
selection-rate SD; their final, normalised score is a different quantity.
Both audits also find inconsistency more frequent when the two candidates are close in quality, measured
through scores or win rates, and \citet{shi-etal-2025-judging} report preferred positions that reverse with
the task. The canonical margin of \cref{sec:factors} is a related but different measure of uncertainty, and
neither identifies missing knowledge as a cause.
In multiple choice, \citet{guda2025quantifying} evaluate per-item and pooled measures together and find
that a mitigation gain depends on the measure: averaged over their models and datasets, PriDe lowers RStd while raising the Permutation Bias Metric (PBM) and CKLD, whereas their
consistency-oriented fine-tuning lowers PBM without lowering CKLD.
The pooling identity makes this combination exact for the pooled selection-rate SD under any balanced
multiple-choice design, with a correction for unbalanced designs and an upper bound for the KL divergence
from uniform.

\paragraph{Gold-dependent measures: RStd and CKLD.}
RStd \citep{zheng2024large} and CKLD \citep{choi-etal-2025-mitigating} compare predictions with gold
labels, through classwise recalls and the gold answer distribution. The pooled statistics we compare
are label-free analogues computed from selection frequencies against uniform. The selection-rate SD factors
through the chosen-position profile $g$; RStd evaluated with every item's correct option balanced across
positions factors through the gold-dependent profile $h$ instead, and canonical-order RStd has no such
identity (\cref{sec:method_metrics}).
\citet{choi-etal-2025-mitigating} observe sharper choice imbalance among incorrect answers and transfer
bias vectors taken from hidden states across datasets. Our partition instead splits the items by whether
their choice stays the same under reordering, which yields an exact identity under balanced orderings, and
we do not identify their hidden-state bias vector with the option-symbol prior.
\paragraph{Corrections and direction.}
PriDe \citep{zheng2024large} estimates a prior on each of a few estimation items and applies the
average of these priors to the remaining items, so its correction is itself pooled. It uses permutations
of the estimation items to cancel the content terms in log probabilities under its multiplicative,
order-invariant model, whereas the pooling identity counts choices and needs no probability model.
Dividing by an estimated shared prior does not enforce agreement across orderings, although exact prior
removal under PriDe's assumptions would recover an order-invariant content distribution.
\citet{zheng-etal-2026-mitigating} train for agreement across orderings with a permutation-aware
GRPO that rewards answers matching the unique most frequent choice over a permutation group. That reward acts on
what the flip rate counts, whereas the PriDe reductions we split come mostly from the skew.

\end{document}